\documentclass{article}
\newif\ifanon
\anonfalse
\ifanon
  \usepackage{tmlr}
\else
  \usepackage[preprint]{tmlr}
\fi
\usepackage{amsmath,amssymb,amsthm}
\usepackage{graphicx}
\usepackage{booktabs}
\usepackage{hyperref}

\newtheorem{proposition}{Proposition}
\newtheorem{lemma}{Lemma}
\newtheorem{remark}{Remark}

\title{Can a Dynamic Internal Field Govern a Transformer's Cognition?\\
Certifiability, not Superiority, in Homeostatic Compute Control}
\author{\name Francisco M. Arrabal-Campos$^{\ast}$ \email fmarrabal@ual.es\\
\addr Department of Chemistry and Physics, Research Centre CIAIMBITAL\\
\addr Department of Engineering, Research Centre CIAIMBITAL\\
\addr University of Almer\'ia, 04120 Almer\'ia, Spain \quad ($^{\ast}$corresponding author)
\AND
\name Francisco G. Montoya \email pagilm@ual.es\\
\addr Department of Engineering, Research Centre CIAIMBITAL\\
\addr University of Almer\'ia, 04120 Almer\'ia, Spain
\AND
\name Alfredo Alcayde \email aalcayde@ual.es\\
\addr Department of Engineering, Research Centre CIAIMBITAL\\
\addr University of Almer\'ia, 04120 Almer\'ia, Spain
\AND
\name Ignacio Fern\'andez \email ifernan@ual.es\\
\addr Department of Chemistry and Physics, Research Centre CIAIMBITAL\\
\addr University of Almer\'ia, 04120 Almer\'ia, Spain}

\ifanon
  \newcommand{\compcite}{\citep{anon2026cognition}}
\else
  \newcommand{\compcite}{\citep{arrabal2026cognition}}
\fi
\newcommand{\coderepo}{\url{https://github.com/fmarrabal/miuracognitive}}
\begin{document}
\maketitle

\begin{abstract}
An intelligent system does not merely reason: it \emph{governs} its own
reasoning---how much to compute, when to stop, which module to activate. We ask
whether that role of \emph{metacognitive governor} can be played by a
\emph{dynamic internal field}: a low-dimensional homeostatic state, with explicit
physics and certified stability, that modulates a transformer's cognition
without performing it. Our instance is the \emph{Homeostatic Background Processor}
(HBP): a field defined over the module graph of a transformer and governed by a
\emph{family} of partial differential equations on the graph Laplacian
---damped wave (forced Klein--Gordon), its diffusive limit, and a KdV-type
dynamics--- that evolves along the iterations of an adaptive-depth
\emph{reasoner} and modulates its computation (halting threshold, block gain,
memory gates; the interface also exposes a router-bias head, not consumed by
the models reported here) through a universal interface of interoception and
modulation. We
characterize the stability of the \emph{integrator} of the whole family, and
we are explicit that this is an integrator certificate and not a closed-loop
one. Two of the three ingredients are classical and we recall them with
attribution: a \emph{placement} dichotomy for antisymmetric operators
(gyroscopic in the second-order branch, positional in the first), which is
Kelvin--Tait--Chetaev \citep{thomson1879treatise,merkin1997} and, in its
first-order half, exactly Anti-Symmetric DGN \citep{gravina2023adgn} with
$\gamma\mathbf{I}$ generalized to $\mathbf{K}$; and a coercivity bound giving
an unconditionally contractive implicit kernel, which is the logarithmic-norm
resolvent estimate \citep{dahlquist1958stability,soderlind2006logarithmic}.
The circulatory branch produces \emph{flutter} with threshold
$\beta\rho(A^3)<2\zeta\omega_0^2$, exact only when stiffness \emph{and}
damping are both multiples of the identity ---a hypothesis none of our runs
satisfies, and under which the threshold is Bottema's criterion
\citep{bottema1955stability}. What is new is narrower, and we now prove it: a discrete
Schur--Cohn criterion with complex coefficients for the Verlet integrator with
velocity coupling, \emph{necessary and sufficient} per latent root and
requiring no commutation hypothesis, which shows that a backward-differenced
gyroscopic term does not inherit the neutrality it has in continuous time.
Empirically ($S_5$, NC$^1$-hard; pre-registered protocol, $n{=}10$, plus a
twenty-seed preregistered deconfounding campaign that partly overturned it),
the answer is threefold:
\textbf{substance no, structure only in part, certifiability yes}.
\emph{Substance (no):} the \emph{type} of the field's physics is
irrelevant for accuracy ---wave, diffusion, gated mixtures, \emph{non-local}
Poisson-type coupling and even a 2D Navier--Stokes flow substrate all give the
same accuracy; the evidence is consistent with the gradient laminating
every modulation demand to a quasi-static set-point, and incompressible
flow hits a physical ceiling
($\nabla\!\cdot u{=}0$ forbids concentrating information). \emph{Structure
(only in part):} the \emph{second} order of the dynamics endows
out-of-distribution \emph{compute allocation} with a robustness the
end-to-end learned halting control (\texttt{gating\_wm}) lacks, and ---after fixing a BF16 bug that froze the physical
parameters--- replicates unchanged; but a preregistered deconfounding
campaign bounds the claim. With twenty fresh seeds, with the first-order
branch rebuilt at \emph{equalized} caps (which the unconditionally
contractive implicit kernel makes safe, so the original caps were an
artifact of the explicit integrator), and with a matched-interface GRU
replacing the physical integrator, the order effect is strong in one
generator family ($+0.087$ $[+0.042,+0.132]$, $t{=}4.0$) and is not detected in
the other ($+0.014$ $[-0.013,+0.040]$, n.s.---an interval that excludes the
original v3 estimate but not a small positive effect): part of the original
contrast was capacity, not order. The GRU governor is statistically
indistinguishable from the field in the first family
($+0.006$ $[-0.051,+0.062]$, n.s.) and nominally exceeds it in the second ($-0.035$
$[-0.067,-0.002]$), with mutual accuracy non-inferiority throughout.
\emph{Certifiability (yes):} what distinguishes the field is therefore not
capability but that its stability is \emph{provable} ---placement
dichotomy, flutter threshold exact under $\mathbf{K}{=}\omega_0^2\mathbf{I}$,
unconditional contraction bound--- and, above all, that the one-step
operator admits an \emph{exact} runtime check. That is a difference of kind,
not of existence: learned recurrences do carry certificates, and for gated
recurrent units in particular there are explicit ISS and incremental-ISS
conditions on the weights \citep{bonassi2021stability}; they are sufficient
and conservative, and obtaining them requires bespoke machinery, whereas the
field's spectral radius is read off directly. We conclude that a dynamic internal field is a viable and
certifiable \emph{compute governor} ---the ``brainstem'' of a cognitive
architecture--- but not an enhancer of cognition: it modulates, it does not
think. A kill-gate with a positively controlled probe finds no evidence for
the remaining route---the field as temporal evidence accumulator---on twelve
frozen solvers of a companion substrate ($\Delta\mathrm{AUC}=+0.0007$
$[-0.0065,+0.0079]$ against a $0.03$ pass threshold); we read this as the
recurrent state already integrating its own history, a mechanism we propose
rather than demonstrate. The null, broad and with a mechanism, delimits which
role a homeostatic field can and cannot play in an adaptive transformer.
\end{abstract}

% =====================================================================
\section{Introduction}
A cognitive architecture is not exhausted by the module that reasons. Alongside
perception, memory and inference, an intelligent system needs a layer that
\emph{governs} its own computation: how much to think before answering, when to
stop, which module to activate, how to distribute effort across parts of the
input of uneven difficulty. In the brain, that function is performed not by the
cortex but by the \emph{neuromodulatory} and homeostatic systems ---brainstem,
hypothalamus, norepinephrine/acetylcholine tone--- which set the gain, the
arousal and the budget of cognition without executing it. This article asks
whether that role of \textbf{metacognitive governor} can be played, in an
adaptive transformer, by a \emph{dynamic internal field}: a low-dimensional
homeostatic state, with explicit physics and certified stability, that
modulates cognition without performing it. (Our substrate is a family of
small decoder-only transformers---$4.2$--$5.6$M parameters, per-variant
counts in the repository---on an
NC$^1$-hard task; the question and the formalism are architecture-generic,
and ``language model'' below refers to the architectural class, not the
scale.)

The concrete motivation is adaptive computation. A standard transformer computes
at fixed depth; under logarithmic precision it is confined to $\mathrm{TC}^0$
\citep{merrill2023parallelism} and only learns ``shortcuts'' that do not
extrapolate \citep{liu2023shortcuts}. Matching depth to difficulty ---Adaptive
Computation Time \citep{graves2016act}, PonderNet \citep{banino2021pondernet},
Universal Transformers \citep{dehghani2019universal}, \emph{looped transformers}
\citep{fan2025looped}, depth-adaptive and early-exit decoding
\citep{elbayad2020depth,schwartz2020right}, and token-level compute routing
\citep{raposo2024mixture}--- is the classic remedy, but in all of them the halting
\emph{policy} is learned end-to-end and can overfit the training distribution.
We propose to anchor that decision in a \emph{dynamic internal field} and ask
what is gained ---and what is not--- by doing so.

Our instance of the governor is the \emph{Homeostatic Background Processor}
(HBP): a low-dimensional field over the model's module graph, governed by a
\emph{family} of PDE operators on the graph Laplacian and advection ---reaction,
diffusion, wave, advection, dispersion and a KdV-type nonlinearity---, which
evolves along the iterations of a recurrent \emph{reasoner} and modulates its
computation through a universal interface: each module exposes interoception
$s_i$ (bottom-up) and receives modulation $m_i$ (top-down), inspired by
biological predictive regulation and neuromodulation
\citep{sterling2012allostasis,barrett2017allostasis,friston2010free,
doya2002metalearning,vecoven2020neuromodulation}
(Fig.~\ref{fig:concept}). The \emph{family} lets us decompose the headline
question into two that the adaptive-computation literature does not isolate:
\emph{(i)~does substance matter}, i.e.\ the \emph{type} of the field's physics
(wave vs.\ diffusion vs.\ KdV, local vs.\ non-local, modulator vs.\ substrate)?
and \emph{(ii)~does structure matter}, i.e.\ the order of the dynamics? We
evaluate with a pre-registered protocol and adversarial verification, and the
answer is threefold: \emph{substance} does not ---the physical regime is
irrelevant for accuracy, a \emph{broad} mechanism null that we explain by
the temporal \emph{coarse-graining} of the gradient and, for the flow substrate,
by an incompressibility limit---; the second-order \emph{structure} does, but
only within one of two generator families once capacity is equalized, and a
matched-interface GRU governor ties the field there and nominally exceeds it
in the other
(Sec.~\ref{sec:v4}); what remains distinctive of the field is neither
substance nor superiority but \emph{certifiability}.

\paragraph{The answer, in one sentence.} A dynamic internal field \emph{can}
govern an LLM's computation ---it is a viable \emph{governor}, with certifiable
stability--- but it \emph{cannot} enhance its cognition: it modulates, it does
not think. And it is not the only thing that can govern: the one learned
alternative we tested---a GRU cell with this same interoceptive
interface---does the job at least as well, matching the field in one generator
family and nominally beating it in the other, so
the field's distinctive contribution is that its stability is provable rather
than merely observed. We thus place
the HBP in the \emph{metacognitive/autonomic} layer of a cognitive architecture
---the computational ``brainstem'', not the cortex--- and characterize
precisely, via a broad null with a mechanism, which role a homeostatic
field can occupy and which it cannot.

\paragraph{Contributions.}
\begin{enumerate}
  \item An architectural \emph{positioning}, backed by evidence: a dynamic
        internal field belongs to the \emph{metacognitive compute-governor}
        layer of an LLM, not to the reasoning layer; the empirical pattern
        (neutral for accuracy; better compute allocation than the
        no-controller baseline, in one of two generator families, and matched
        there by a learned recurrent governor) is what suggests that
        placement (Sec.~\ref{sec:hbp}, \ref{sec:null}).
  \item A novel formalism: a \emph{family} of homeostatic PDE fields over the
        module graph (wave, diffusion, advection, dispersion, saturated KdV
        nonlinearity), with continuous mixing between first- and second-order
        branches and optional \emph{gating} of the physics by interoception
        (Sec.~\ref{sec:hbp}).
  \item A stability characterization of the family's \emph{integrator},
        assembled from classical results with attribution and one new
        criterion. Classical: the placement dichotomy
        (Kelvin--Tait--Chetaev \citep{thomson1879treatise,merkin1997}; its
        first-order half is A-DGN \citep{gravina2023adgn} with
        $\gamma\mathbf{I}\to\mathbf{K}$), the flutter threshold (Bottema's
        criterion \citep{bottema1955stability} at isotropic stiffness and
        damping, and the internal-damping whirl threshold of
        \citet{smith1933motion}), and the unconditionally contractive
        implicit kernel (the logarithmic-norm resolvent bound
        \citep{dahlquist1958stability,soderlind2006logarithmic,bai2003hermitian}).
        New, and proved here (App.~\ref{app:proofs},
        Prop.~\ref{prop:verlet-proof}): a discrete complex Schur--Cohn
        condition for the Verlet integrator with velocity coupling, necessary
        \emph{and} sufficient per latent root and free of any commutation
        hypothesis, with the passage to an operator-level box certificate
        quantified rather than assumed. All are imposed as a
        differentiable penalty during training and audited on trained
        checkpoints with the exact spectral radius
        (Sec.~\ref{sec:theory}, App.~\ref{app:scope}).
  \item Two pre-registered empirical studies on state tracking in $S_5$: a
        first campaign ($n{=}10$) in which the OOD compute-robustness effect
        survives a paired re-run with learnable physics after a
        numerical-precision bug (BF16 freezing), and a twenty-seed
        \emph{deconfounding} campaign that bounds the claim---at equalized caps
        the order effect holds in one generator family and not the other, and a
        matched-interface GRU governor attains it too
        (Secs.~\ref{sec:exp},~\ref{sec:v4}).
  \item A \emph{broad} mechanism null, adversarially verified: none of the members
        of the family moves accuracy, not even in dual tasks designed to excite
        regime switching, nor under \emph{non-local} coupling (Poisson
        enslavement on the graph), nor with the field as a \emph{substrate} of a
        2D Navier--Stokes flow ---where an incompressibility limit
        ($\nabla\!\cdot u{=}0$) prevents delivering information to a point.
        We propose the temporal \emph{coarse-graining} of the gradient as the
        explanation for the modulator null (Sec.~\ref{sec:null}).
\end{enumerate}

% =====================================================================
\section{Related work}
\label{sec:related}

\paragraph{Adaptive computation and learned \emph{halting}.}
The idea of decoupling compute depth from architectural depth goes back to
Adaptive Computation Time (ACT) \citep{graves2016act}; PonderNet
\citep{banino2021pondernet} reformulates it probabilistically, and Universal
Transformers \citep{dehghani2019universal} and \emph{looped transformers}
\citep{fan2025looped,giannou2023looped} carry it to depth recurrence. Our
\emph{reasoner} inherits this paradigm with PonderNet-style halting. The crucial
distinction is that in these works the compute-allocation policy is learned
end-to-end; our study provides direct evidence that such a policy can overfit
(the adaptivity of the \texttt{gating\_wm} control collapses out of distribution
in one of the two generator sets) and that anchoring it in a stateful governor
with an interoceptive interface makes it more robust---though the
deconfounding of Sec.~\ref{sec:v4} shows that the \emph{explicit physics} is
not what buys the robustness: a GRU with the same interface is as robust in
\texttt{adjacent} and nominally more so in \texttt{cycle\_transp}.

\paragraph{Expressivity and state tracking.}
Under logarithmic precision, a fixed-depth transformer sits in $\mathrm{TC}^0$
\citep{merrill2023parallelism}; the word problem over $S_5$ is NC$^1$-complete by
the theorem of \citet{barrington1989}, and transformers learn shortcuts that do not
extrapolate \citep{liu2023shortcuts}. Intermediate generation extends the
expressive power \citep{merrill2024cot}, and state-space models share the
limitation \citep{merrill2024illusion}. This justifies that our task
\emph{requires} iterative computation and motivates the extrapolation axis
\citep{anil2022length,jelassi2023arithmetic}.

\paragraph{External and working memory.}
Neural Turing Machines \citep{graves2014ntm} and the DNC \citep{graves2016dnc}
enable algorithms with addressable memory. Our \texttt{gating\_wm} control
incorporates a working memory in this spirit; our experiments attribute the
in-distribution accuracy gain to the \emph{recurrence} rather than to the
memory (\texttt{gating} $\ge$ \texttt{gating\_wm} in both generator sets) and,
in either case, not to the HBP: a negative that we
make explicit.

\paragraph{Certified learned dynamics, and where our certificate sits.}
An earlier version of this paper did not engage this literature, which is a
serious omission because it is where the bar actually is. Stability guarantees
for learned recurrent models come in two grades. \emph{By construction:}
coRNN and UnICORNN \citep{rusch2021cornn,rusch2021unicornn} are
structure-preserving discretizations of second-order oscillator networks with
proven state and gradient bounds; AntisymmetricRNN
\citep{chang2019antisymmetricrnn} and A-DGN \citep{gravina2023adgn} obtain
stability from antisymmetric parameterization; Recurrent Equilibrium Networks
\citep{revay2024recurrent} give a free parameterization of \emph{all} models
that are contracting and satisfy prescribed incremental IQCs, trainable by
unconstrained gradient descent; Lipschitz RNNs \citep{erichson2021lipschitz}
and the Lyapunov-projected models of \citet{kolter2019learning} do the same by
other routes; LinOSS \citep{rusch2025oscillatory} and its damped extension
\citep{boyer2025dissipate} obtain stable oscillatory state-space layers with a
non-negativity condition, the latter decoupling damping from frequency exactly
as our learnable per-dimension $\zeta$ does; and CON \citep{stolzle2024con}
proves global asymptotic stability \emph{and} input-to-state stability for a
coupled damped-oscillator network in closed loop. \emph{By certificate:}
\citet{bonassi2021stability} give explicit ISS and incremental-ISS conditions
on the weights of a gated recurrent unit, checkable post hoc or imposable
during training.

Two consequences for this paper, both uncomfortable and both stated here
rather than left for a referee. First, our guarantee is of the weaker grade:
we impose a differentiable penalty and probe the operator at evaluation time,
where coRNN, REN and D-LinOSS are stable \emph{by construction}. Second, the
closed-loop result we do not prove ---the field, its host, and the
interoception/modulation path considered as one system--- is exactly what CON
proves for its setting. What is genuinely unoccupied in this literature is the
conjunction of second order, instance-gated (time-varying) coefficients, and
closed loop: REN has the loop without the first two, CON has the first and
third without gating, LinOSS has only the first. We do not fill that gap here;
we mark it.

\paragraph{PDE-governed neural networks on graphs.}
Interpreting layers as discretizations of continuous dynamics is the basis of
Neural ODEs \citep{chen2018node}; on graphs, PDE-GCN \citep{chamberlain2021grand,eliasof2021pdegcn}
derives architectures from the diffusion and wave equations, GraphCON
\citep{rusch2022graphcon} couples oscillators at the nodes (the second-order
dynamics closest to our wave branch), ADR-GNN \citep{eliasof2023adr} adds
advection and reaction, and Anti-Symmetric DGN \citep{gravina2023antisymmetric}
obtains stability by construction with antisymmetric operators. We should be
explicit about how close that last one is: \textbf{the first-order half of our
placement dichotomy \emph{is} A-DGN}, with their $\gamma\mathbf{I}$ generalized
to a symmetric $\mathbf{K}\succeq\omega_0^2\mathbf{I}$. Claiming the
generalization is honest; not flagging the coincidence would not be. Our family
differs on three points: (i)~the field does not transport the task
representation but a low-dimensional \emph{modulatory} state; (ii)~the ``time''
of the PDE is the reasoner's iterations, not the layers; (iii)~the
wave$\leftrightarrow$diffusion mixing and the odd operators (advection, KdV
dispersion) require a stability analysis that the GNN literature does not cover:
incorrect placement of an antisymmetric operator in a second-order dynamics
produces circulatory \emph{flutter}, a classical phenomenon of non-conservative
mechanics \citep{ziegler1952,merkin1997,kirillov2013} that, to our knowledge, had
not been pointed out in this context.

\paragraph{Homeostasis and neuromodulation.}
Differentiable plasticity and neuromodulation
\citep{miconi2018plasticity,miconi2019backpropamine} enable internal modulation
of computation; homeostatic variables coupled to vulnerability signals confer
adaptability under \emph{concept shift} \citep{man2022homeostatic}. The HBP
articulates this intuition as a physically motivated and analyzable dynamical
system, with certified stability and its effect on the compute policy isolated
experimentally.

% =====================================================================
\section{A family of homeostatic fields}
\label{sec:hbp}

\begin{figure}[t]\centering
\includegraphics[width=\textwidth]{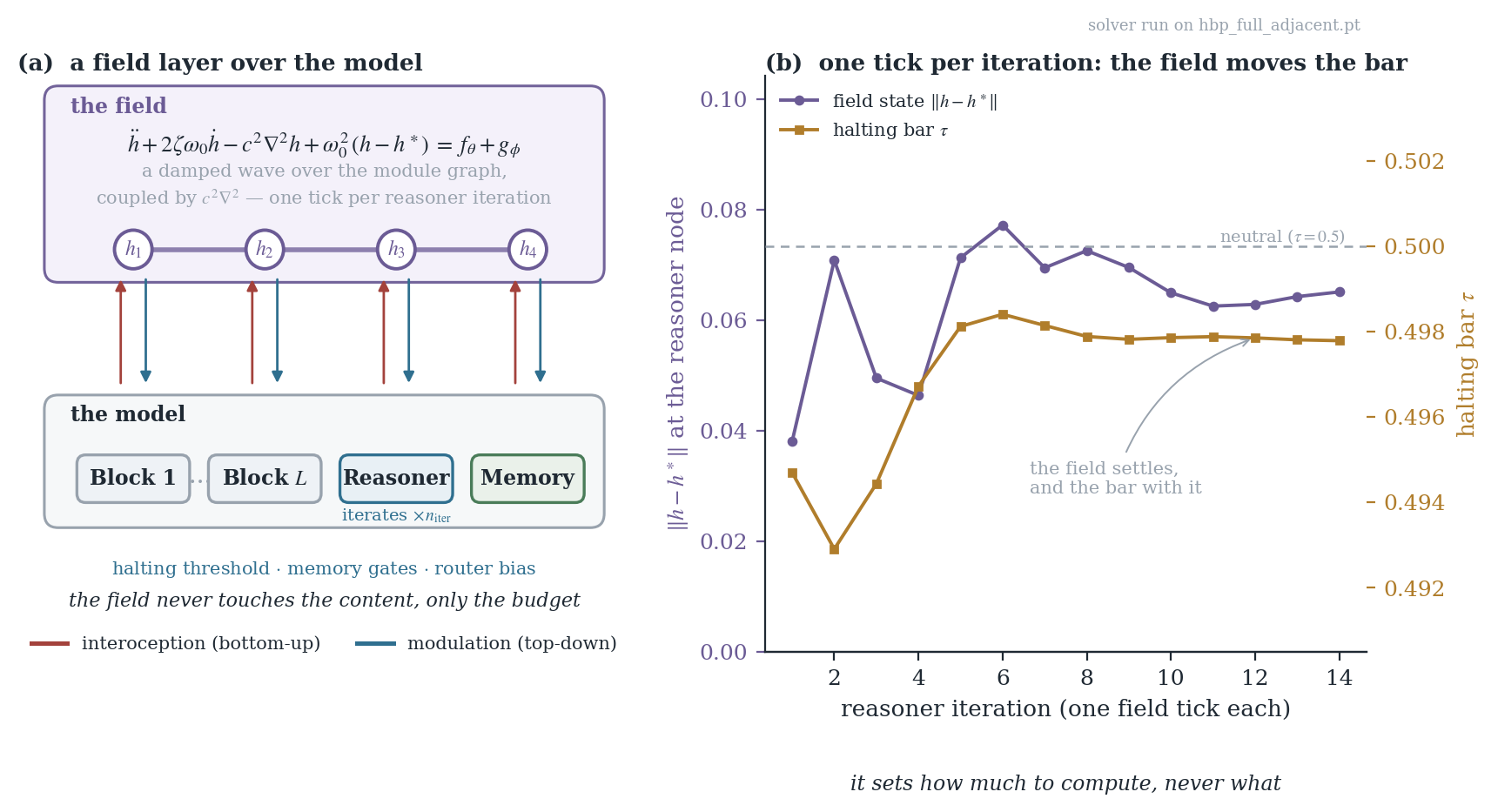}
\caption{The \emph{Homeostatic Background Processor}. An internal-state field
$h$ (top) lives over the model's module graph ---backbone blocks, the
adaptive-depth reasoner and the working memory (bottom)--- coupled by the graph
Laplacian ($-c^2\nabla^2 h$). Each module exposes an \emph{interoception} signal
$s_i$ (bottom-up) and receives \emph{modulation} $m_i$ (top-down: halting
threshold, block gain, memory gates; the interface also exposes a router-bias
head, not consumed by the experiments reported here). The field evolves under a damped-wave PDE
(one instance of the family of Sec.~\ref{sec:hbp}), one tick per reasoner
iteration.}
\label{fig:concept}
\end{figure}

Let there be a graph of $N$ nodes (modules: $L$ backbone blocks, the reasoner
and the working memory) with combinatorial Laplacian
$\mathbf{L}\in\mathbb{R}^{N\times N}$ (symmetric, PSD) and oriented advection
matrix $\mathbf{A}$ (antisymmetric; spectrum $\{\pm i\mu_k\}$), defined over the
chain backbone${}\to{}$reasoner${}\to{}$memory. Each node holds an
\emph{internal state vector} (VEI) $h_i\in\mathbb{R}^{d_h}$; we write
$u:=h-h^\ast$ (deviation from the learnable rest state). Per latent dimension, we
define the operators
\begin{equation}
\mathbf{K}:=\omega_0^2\mathbf{I}+c^2\mathbf{L},\qquad
\mathbf{C}:=2\zeta\omega_0\mathbf{I}+D\,\mathbf{L},\qquad
\mathbf{G}:=b\,\mathbf{A}+\beta\,\mathbf{A}^3,
\end{equation}
(stiffness: reaction $+$ spatial diffusion; dissipation: uniform $+$ structural;
antisymmetric: advection $+$ third-order \emph{dispersion}, the only odd spatial
operator well defined over an oriented graph, with $(i\mu)^3=-i\mu^3$), and the
KdV-type nonlinear perturbation
\begin{equation}
r(u):=-\nu\,\tanh(u)\odot\tanh(\mathbf{A}u),
\end{equation}
the \emph{saturated} $u\,\partial_x u$ term: $r(0)=0$ with null Jacobian at the
equilibrium (same quadratic order as KdV near $u=0$) and globally bounded by
$\nu$. The family consists of two branches coupled by a convex mixture
$\alpha\in[0,1]$:
\begin{align}
\text{(wave, 2nd order):}\quad &
\ddot u + (\mathbf{C}+\mathbf{G})\,\dot u + \mathbf{K}u
   = r(u) + f_\theta(h,s) + g_\phi(h,x),
\label{eq:wave}\\
\text{(diffusion, 1st order):}\quad &
\gamma\,\dot u + (\mathbf{K}+\mathbf{G})\,u
   = r(u) + f_\theta(h,s) + g_\phi(h,x),
\label{eq:diff}
\end{align}
where $f_\theta$ is a bounded self-check ($\tanh$, gain $\le 0.3$) that consumes
the interoception $s$, $g_\phi$ a bounded external forcing from the input, and
$\gamma$ a learnable rate of the diffusive branch (decoupled from $\zeta$). Note
the distinct \emph{placement} of $\mathbf{G}$: gyroscopic (on $\dot u$) in
\eqref{eq:wave}, positional (on $u$) in \eqref{eq:diff};
Lemma~\ref{lem:placement} shows it is the only stable choice in each branch. The
state advances by
\begin{equation}
h_{n+1}=\alpha\, \Phi_{\mathrm{wave}}(h_n,h_{n-1})
        +(1-\alpha)\,\Phi_{\mathrm{diff}}(h_n),
\end{equation}
with $\Phi_{\mathrm{wave}}$ a position Verlet step and $\Phi_{\mathrm{diff}}$ an
IMEX \emph{backward-Euler} step (stiff part $\mathbf{K}+\mathbf{G}$ implicit;
Prop.~\ref{prop:imex}). This differs from the classical implicit--explicit
split of \citet{ascher1995imex}, where the stiff \emph{symmetric} diffusive
term is taken implicitly and the advective, antisymmetric term explicitly ---
with the conditional stability that entails. Here the antisymmetric operator
goes \emph{inside} the implicit kernel, which is what the coercivity argument
of Prop.~\ref{prop:imex} makes unconditionally safe.
The coefficient $\alpha$ can be (i) an architectural
constant ($\alpha{=}1$: wave; $\alpha{=}0$ or order 1: diffusion), (ii) imposed
per instance (experimental oracles), or (iii) \emph{gated by interoception}
together with $D$ and $b$ (``homeostasis chooses its physics''), the hypothesis
that Section~\ref{sec:null} evaluates and refutes for accuracy.

\paragraph{Instantiated physics.} We evaluate three members of the family:
\texttt{hbp\_full} (damped wave: $\alpha{=}1$, $\mathbf{G}{=}0$, $D{=}0$),
\texttt{hbp\_first} (diffusive relaxation: order 1, overdamped limit), and
\texttt{hbp\_kdv} (wave $+$ dispersion $+$ KdV nonlinearity: $\alpha{=}1$,
$\beta_{\max}{=}0.1$, $\nu_{\max}{=}0.3$); plus the gated variant
\texttt{hbp\_mix}. At $N{=}6$ the dispersion has only three pairs of modes:
there is no solitonic regime, and we make it explicit as a limit of the
testbed.

\paragraph{``Time'' is thinking.} The HBP ticks once per reasoner iteration
(PonderNet-style halting, up to $N_{\max}{=}24$). At iteration $n$: (i) the
field modulates the reasoner block with a soft perturbation
$1+s_{\mathrm{mod}}\tanh(\cdot)$ around the identity; (ii) per-node interoception
$s_n$ is collected (progress, activation, effort, attention entropy, valid
fraction); (iii) one step of the family is integrated; (iv) the VEI biases the
halting threshold and the write/forget gates of the working memory. The physical
parameters ($\omega_0,\zeta,c,D,b,\beta,\nu,\gamma$) are learnable per latent
dimension within safe ranges (squash), and are stored in FP32
(Sec.~\ref{sec:bf16}).

% =====================================================================
\section{Stability analysis of the family}
\label{sec:theory}

We define the per-dimension energy
$E(u,\dot u)=\tfrac12\lVert\dot u\rVert^2+\tfrac12\,u^\top\mathbf{K}u$
(kinetic $+$ restoring well $+$ graph Dirichlet energy). In the autonomous case
of \eqref{eq:wave} with $\mathbf{G}=0$,
$\dot E=-\dot u^\top\mathbf{C}\dot u\le 0$ (strict dissipation with
$\zeta\ge\zeta_{\min}>0$). The nontrivial question is where the
\emph{antisymmetric} operators can enter without destroying this structure.

\begin{lemma}[Placement of antisymmetric operators]
\label{lem:placement}
Let $\mathbf{K}=\mathbf{K}^\top\succ0$, $\mathbf{C}=\mathbf{C}^\top\succ0$ and
$\mathbf{G}=-\mathbf{G}^\top$ with spectrum $\{i\mu_k\}$. Then:
\begin{itemize}
  \item[(a)] \textbf{(gyroscopic: safe).} For
  $\ddot u+(\mathbf{C}+\mathbf{G})\dot u+\mathbf{K}u=0$, the energy satisfies
  $\dot E=-\dot u^\top\mathbf{C}\dot u\le0$, since
  $\dot u^\top\mathbf{G}\dot u=0$ (the gyroscopic force does no work); by
  LaSalle, the origin is globally and asymptotically stable \emph{for all}
  $b,\beta$ (Kelvin--Tait--Chetaev theorem \citep{merkin1997,kirillov2013}).
  \item[(b)] \textbf{(positional in first order: safe).} For
  $\gamma\dot u+(\mathbf{K}+\mathbf{G})u=0$, every eigenvalue $\lambda$ of
  $-(\mathbf{K}+\mathbf{G})/\gamma$ satisfies
  $\operatorname{Re}\lambda=-(v^*\mathbf{K}v)/\gamma\le-\omega_0^2/\gamma<0$
  unconditionally in $\lVert\mathbf{G}\rVert$ (numerical range:
  $v^*\mathbf{G}v\in i\mathbb{R}$). With $\mathbf{K}=0$,
  $\sigma(-\mathbf{G})\subset i\mathbb{R}$: conservative transport and
  dispersion, with the KdV dispersion relation on the graph.
  \item[(c)] \textbf{(circulatory: flutter, with a threshold that is exact only
  in the isotropic case).} For
  $\ddot u+\mathbf{C}\dot u+(\mathbf{K}+\mathbf{G})u=0$ with
  $\mathbf{K}=\omega_0^2\mathbf{I}$ \emph{and} $\mathbf{C}=2\zeta\omega_0\mathbf{I}$
  (equal frequencies and isotropic damping: the degenerate Merkin
  case \citep{merkin1997}; this is the per-dimension structure of the field only
  when $c{=}0$ and there is no structural diffusion, which no run of ours
  satisfies ---see App.~\ref{app:scope}, where we also note that the headline
  arm sets $\mathbf{G}=0$ outright, making this a design-exclusion result
  rather than a certificate of a system we ran) and
  $\mathbf{G}=\beta\mathbf{A}^3$, the system decouples in the unitary basis of
  $\mathbf{G}$ and the Routh--Hurwitz criterion with complex coefficients gives
  asymptotic stability \emph{if and only if}
  $\beta\,\rho(\mathbf{A}^3)<2\zeta\omega_0^2$. Under these hypotheses this is
  Bottema's criterion \citep{bottema1955stability} evaluated at isotropic
  stiffness and damping, and it coincides with the internal-damping whirl
  threshold of \citet{smith1933motion}; we recall it rather than claim it. The
  threshold vanishes with the
  damping (as $\zeta\to0$, every $\beta>0$ destabilizes \citep{ziegler1952}) and
  grows only linearly with $\zeta\omega_0^2$. Note that the isotropic case is
  the one in which the celebrated destabilization paradox does \emph{not}
  appear: the threshold tends to zero continuously with $\zeta$, with none of
  the discontinuous drop that makes the general case interesting.
  \item[(d)] \textbf{(nonlinear perturbation).} $r$ satisfies $r(0)=0$,
  $Dr(0)=0$ and $\lVert r(u)\rVert\le\nu$: the linearization at the equilibrium
  ---and with it (a)--(c) and the discrete certificates--- is not altered; in
  \eqref{eq:diff} the origin is GAS if $\nu(1+\rho(\mathbf{A}))<\omega_0^2$
  ---the Jacobian of the saturating perturbation contributes a diagonal
  \emph{and} an $\mathbf{A}$-coupled term, so the Lipschitz constant is
  $\nu(1+\rho(\mathbf{A}))$, not $\nu\rho(\mathbf{A})$ as an earlier version
  of this lemma stated; outside that condition we claim only local stability
  and do not exhibit the basin. On the invariant set induced by the
  state projection ---a hard clamp in the implementation--- $r+f_\theta+g_\phi$
  acts as a bounded forcing and both branches are BIBS.
\end{itemize}
\end{lemma}

Lemma~\ref{lem:placement} justifies the placement of \eqref{eq:wave} and
\eqref{eq:diff} and invalidates the circulatory alternative: at our operating
point ($\beta_{\max}{=}0.1$, $\rho(\mathbf{A}^3){=}5.85$) the threshold of (c),
used here as the design guide it is for $c>0$ (App.~\ref{app:scope}), is
violated over almost the entire admissible parameter box, and the free
simulation confirms the flutter (divergence) against the gyroscopic decay.

\begin{proposition}[Discrete stability of the wave branch]
\label{prop:verlet}
For the position Verlet
$u_{t+1}=(2\mathbf{I}-\Delta t^2\mathbf{K}-\Delta t(\mathbf{C}+\mathbf{G}))u_t
-(\mathbf{I}-\Delta t(\mathbf{C}+\mathbf{G}))u_{t-1}$, every latent eigenvalue
$z$ solves the complex scalar polynomial
$z^2-(2-q-g-i\tilde\mu)z+(1-g-i\tilde\mu)=0$ with
$q=\Delta t^2\,x^*\mathbf{K}x$, $g=\Delta t\,x^*\mathbf{C}x$ real and
$\tilde\mu=\Delta t\,\mathrm{Im}(x^*\mathbf{G}x)$,
$|\tilde\mu|\le\Delta t\,\rho(\mathbf{G})$,
$\rho(\mathbf{G})=\max_k|b\mu_k-\beta\mu_k^3|$. For $q>0$, Cohn's criterion
for the complex quadratic gives $|z|<1$ \emph{if and only if}
\begin{equation}\label{eq:schurcohn}
\text{(i)}\;\;\tilde\mu^2< g(2-g),\qquad
\text{(ii)}\;\;q\,(g^2+\tilde\mu^2)< 2g\,\big(g(2-g)-\tilde\mu^2\big).
\end{equation}
Two remarks on scope, both proved in App.~\ref{app:proofs}. First, the
reduction to \eqref{eq:schurcohn} is a \emph{latent-root} argument and
therefore needs \textbf{no commutation hypothesis}: each eigenvalue satisfies
the scalar polynomial with the Rayleigh quotients of its own latent vector,
whether or not $\mathbf{K}$, $\mathbf{C}$ and $\mathbf{G}$ are simultaneously
diagonalizable. Second, the equivalence is per latent root; certifying the
whole operator requires the triples $(q,g,\tilde\mu)$ to be controlled, and
bounding them by the spectra of $\mathbf{K},\mathbf{C},\mathbf{G}$ replaces
their joint numerical range by a \emph{box} that contains it, so the box form
is sufficient but not necessary ---and, we find, markedly conservative.
At $\tilde\mu=0$, (ii) reduces to the classical damped-Verlet criterion
$q<4(1-\zeta\omega_0\Delta t)$. Condition (i) is \emph{genuinely discrete}: the
gyroscopic term discretized with a backward difference does not inherit the
continuous neutrality, but injects energy at the exact rate
$\sqrt{1+\tilde\mu^2}$ per step, which the dissipation must absorb; as
$\Delta t\to0$ the condition empties ($\tilde\mu^2=O(\Delta t^2)$ against
$g(2-g)=O(\Delta t)$) and the continuum is recovered. A Cayley discretization of
the gyroscopic term \citep{iserles2000lie} would restore exact neutrality at the
cost of an implicit step.
\end{proposition}

This is the one result in this paper for which we have not found a precedent;
the placement dichotomy and the contraction bound are classical (App.~B). We
prove it in App.~\ref{app:proofs} and check it in
\texttt{experiments/verify\_verlet\_schurcohn.py}: the symbolic identity, $302\,400$ cells
of a $(q,g,\tilde\mu)$ grid against exact roots with zero discrepancies, and
$400$ draws of non-commuting $(\mathbf{K},\mathbf{C},\mathbf{G})$ ---none of
the $400$ pairs commuted--- with a maximum latent-polynomial residual of
$6\cdot10^{-14}$.

\begin{proposition}[IMEX diffusive branch with implicit antisymmetric kernel]
\label{prop:imex}
With the antisymmetric part inside the implicit kernel,
$\mathbf{M}_G:=\mathbf{I}+\lambda(\mathbf{K}+\mathbf{G})$, $\lambda=\Delta t/\gamma$,
one has $\operatorname{Re}\langle x,\mathbf{M}_Gx\rangle\ge(1+\lambda\omega_0^2)\lVert x\rVert^2$
for all $x$, hence $\lVert\mathbf{M}_G^{-1}\rVert_2\le(1+\lambda\omega_0^2)^{-1}<1$
\emph{unconditionally} in $\gamma,\omega_0,c,b,\beta$ and with no commutation
hypothesis between $\mathbf{L}$ and $\mathbf{G}$. (With $\mathbf{G}$ explicit
this claim is false; the Cayley variant on the antisymmetric part preserves
exactly the norm of the conservative transport.)
\end{proposition}

\begin{remark}
The force $-\mathbf{G}\dot u$ in the wave branch is of Coriolis type: it
preserves the stability structure of the antisymmetric operators, not their
literal transport semantics, which survive only in the first-order branch.
\end{remark}

\begin{remark}
The stability of the convex mixture $\alpha$ of the discrete maps does not
follow from the continuous lemma (the per-branch Lyapunov functions are
distinct), and ---correcting an earlier version of this remark--- neither does
it follow from Propositions~\ref{prop:verlet} and~\ref{prop:imex}: the
spectral radius is not convex, so per-branch bounds do not transfer to the
mixture, and a norm bound cannot be combined with a spectral-radius bound for
a non-normal companion. The runtime probe builds the companion of the wave
branch only. \textbf{The gated mixture arm therefore has neither an analytic
nor an exact numerical certificate}, and we report its results on that
footing. We attempted to repair this with the standard instrument ---a common
quadratic Lyapunov function obtained by semidefinite programming over the
vertices of the coefficient box--- and report the outcome in
App.~\ref{app:lmi}: it succeeds for the wave branch on the region the trained
models occupy, and it fails for the mixture. The per-branch conditions are
imposed as a \emph{differentiable penalty} during training (on the learned
coefficients, or their caps if gated) and verified at evaluation time with the
exact spectral radius of the $2N{\times}2N$ \emph{companion} matrix per dimension
(covers the non-normal case $[\mathbf{L},\mathbf{A}]\ne0$). Under bounded
forcing, the VEI is expected to be ISS to a ball around $h^\ast$ by a standard
argument \citep{sontag1989smooth}, which we do not prove here and on which
nothing below depends. The energy identity, the flutter threshold and the
discrete criterion were validated numerically against
the implemented operators at fixed parameter draws (energy identity to residual $4{\cdot}10^{-16}$;
flutter threshold to $\pm2\%$; $0/2\cdot10^5$ discrepancies of the discrete
criterion against exact roots).
\end{remark}

\begin{figure}[t]\centering
\includegraphics[width=\textwidth]{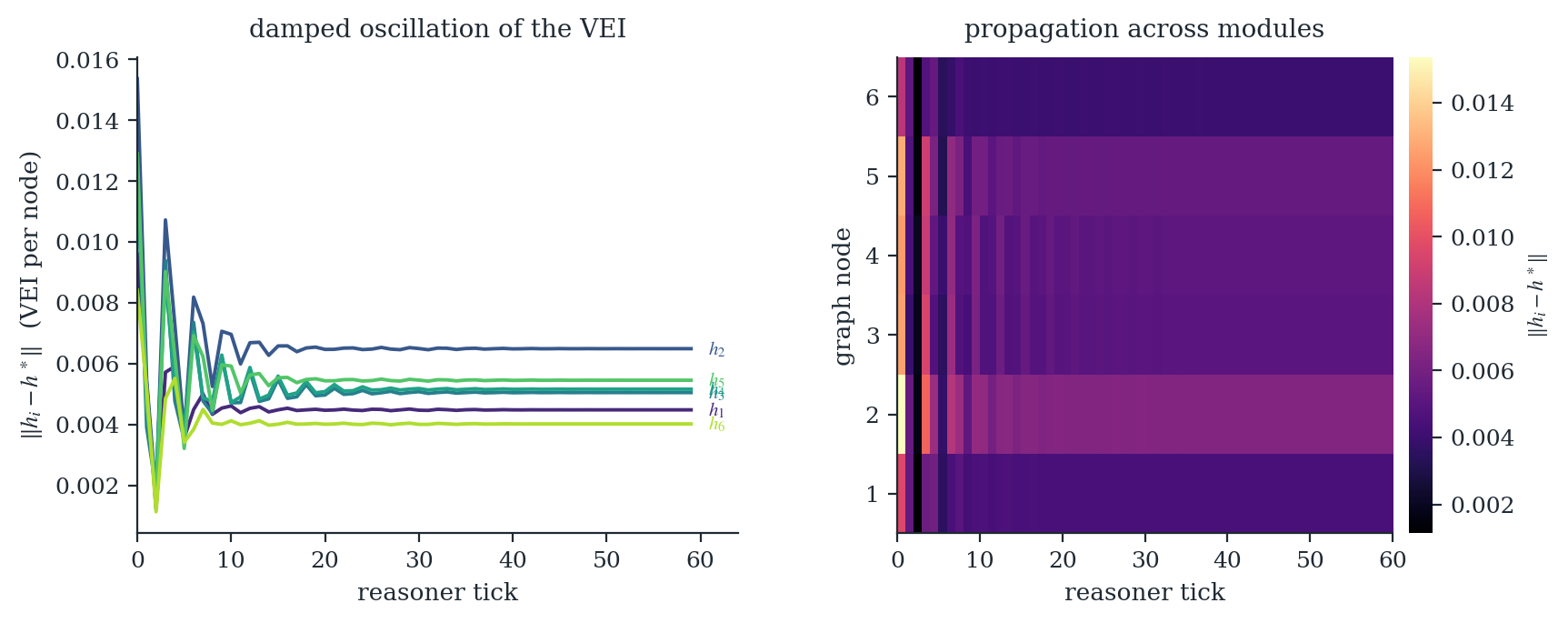}
\caption{Intrinsic dynamics of the field (simulation of the solver of
Sec.~\ref{sec:hbp}). An interoceptive impulse at node~1 induces a \emph{damped
oscillation} of the VEI (left, $\|h_i-h^\ast\|$ per node) that \emph{propagates}
to neighboring modules through the Laplacian coupling (right,
node${\times}$tick map). It is the second-order regime that the analogy with
coupled spins in NMR describes (a local impulse spreads and precesses); the
Verlet scheme integrates it stably under the certificate of
Prop.~\ref{prop:verlet}.}
\label{fig:dynamics}
\end{figure}

\paragraph{NMR analogy.} The family is formally that of a system of coupled
spins: $\zeta\sim1/T_2$ (transverse relaxation), $\omega_0$ (precession),
$c^2\mathbf{L}$ (dipolar coupling), $D\sim1/T_1$ (longitudinal relaxation / spin
diffusion), $b\,\mathbf{A}$ (directed transport) and $\beta\,\mathbf{A}^3$
(dispersion: mode-dependent group velocity; Fig.~\ref{fig:dynamics} shows the
real propagation-oscillation). The normal modes of $\mathbf{L}$ are the spectrum
of the system, and the flutter certificate of Lemma~\ref{lem:placement}(c) is
the analog of the instability threshold of a spin system pumped above its
relaxation rate.

% =====================================================================
\section{Experiments}
\label{sec:exp}

\paragraph{Task and protocol.} Non-commutative state tracking: composition of
$K$ generators of $S_5$ (NC$^1$-hard $\Rightarrow$ requires recurrence), with
dense supervision of the running permutation and two generator sets (adjacent
transpositions; 5-cycle $+$ transposition). \emph{Pre-registered} protocol
(declared before looking at any number): fresh seeds $10..19$, evaluation with a
disjoint seed (no contamination), argmax restricted to the answer sub-vocabulary,
$N_{\max}{=}24\ge K_{\max}$, primary analysis paired by seed (Fisher-$z$ of the
pure-OOD correlation, $K\ge14$) with Holm correction over the 4 contrasts. Models
of $4.2$--$5.6$M parameters (\texttt{vanilla} $4.23$M, \texttt{gating} $5.28$M,
\texttt{gating\_wm} $5.54$M, the field variants $5.60$M; per-variant counts in
the repository). OOD regime: training with $K\le12$, evaluation up
to $K\le24$.

\subsection{A numerical-precision bug and its A/B: the effect is structural}
\label{sec:bf16}
During the audit we discovered that, in BF16, the field's raw physical parameters
($|\theta|\sim0.3$--$1.6$) have $\mathrm{ULP}/2$ ($1.0{\cdot}10^{-3}$ to
$3.9{\cdot}10^{-3}$) \emph{larger} than the typical Adam step
($\mathrm{lr}\approx3{\cdot}10^{-4}$): the update is rounded to zero at every
step and $\omega_0,\zeta,c$ remain \textbf{frozen at their initialization}
throughout training (verified: $\zeta\equiv0.5$ exact after 2500 steps; with the
fix ---anchoring those parameters to FP32--- they move). With the bug fixed, we
re-ran the full protocol (90 cells). The seed-paired A/B between frozen and
learnable physics gives
$\Delta\,\mathrm{corr}_{\mathrm{OOD}}=+0.010\pm0.046$ (adjacent) and
$+0.019\pm0.054$ (5-cycle), both non-significant at $n{=}10$ and with intervals
wide enough to contain effects the size of the order effect itself:
\textbf{we find no evidence that the HBP effect depends on the fine tuning of
its physical constants}. This A/B re-runs the same cells and seeds under a
different numerical precision, so it is a paired same-batch contrast rather
than an independent replication ---and Sec.~\ref{sec:v4} then bounds the
structural reading to one generator family.
We report both versions for transparency; all numbers that follow are
from the fixed version (v3, learnable physics).

\subsection{In distribution: recurrence---not the field, and not the working memory---explains accuracy}
\begin{table}[h]\centering
\caption{Accuracy by difficulty (in-distribution, mean over $n{=}3$ seeds;
per-seed values and paired contrasts in the repository).}
\begin{tabular}{lcccccc}
\toprule
& \multicolumn{3}{c}{$S_5$ adjacent} & \multicolumn{3}{c}{$S_5$ 5-cycle$+$transp.}\\
Model & short & medium & long & short & medium & long \\
\midrule
vanilla     & $0.984$ & $0.670$ & $0.147$ & $0.974$ & $0.976$ & $0.581$ \\
gating      & $0.985$ & $0.761$ & $0.218$ & $0.974$ & $0.980$ & $0.725$ \\
gating\_wm  & $0.983$ & $0.747$ & $0.197$ & $0.974$ & $0.982$ & $0.714$ \\
hbp\_first  & $0.987$ & $0.776$ & $0.245$ & $0.974$ & $0.985$ & $0.730$ \\
hbp\_full   & $0.986$ & $0.771$ & $0.248$ & $0.974$ & $0.984$ & $0.741$ \\
\bottomrule
\end{tabular}
\end{table}
The gain decomposes cleanly, and the working memory is not where it comes
from: recurrence buys $+0.071$ (adjacent) and $+0.144$ (5-cycle) in the long
stratum (\texttt{vanilla}$\to$\texttt{gating}), adding the working memory buys
nothing and is nominally negative (\texttt{gating}$\to$\texttt{gating\_wm}:
$-0.021$ and $-0.011$), and the HBP adds
$+0.03$--$0.05$ in the long stratum, which does not reach significance at
$n{=}3$ ($p\ge0.11$). The 2nd- vs 1st-order contrast is null in distribution.

\subsection{Out of distribution (v3, before deconfounding): a second-order adaptivity signal}
We measure compute adaptivity
$\mathrm{corr}(K,\mathbb{E}[n_{\mathrm{iter}}])$ restricted to the pure-OOD range
($K\ge14$), paired by seed (Table~\ref{tab:ood}, Fig.~\ref{fig:ood}).

\begin{table}[h]\centering
\caption{Pure-OOD compute adaptivity, protocol v3 (learnable physics), $n{=}10$
fresh seeds ($\pm$std). Fisher-$z$ paired contrasts against the
\texttt{gating\_wm} control. Boldface marks \emph{uncorrected} $p<0.05$: neither
contrast survives Holm-4 individually, and the order contrast these numbers
support is bounded by the deconfounding of Sec.~\ref{sec:v4}.}
\label{tab:ood}
\begin{tabular}{lccc}
\toprule
& gating\_wm (WM) & hbp\_first (diffusive) & hbp\_full (wave) \\
\midrule
$S_5$ adjacent        & $0.002\pm0.087$ & $0.036\pm0.090$ & $0.143\pm0.134$ \\
\quad vs WM           & --  & $t{=}0.77$, $p{=}0.46$ & $\mathbf{t{=}2.68}$, $p{=}0.025$ \\
$S_5$ 5-cycle$+$transp. & $0.213\pm0.068$ & $0.185\pm0.077$ & $0.268\pm0.085$ \\
\quad vs WM           & --  & $t{=}{-}1.42$, $p{=}0.19$ & $\mathbf{t{=}2.62}$, $p{=}0.028$ \\
\bottomrule
\end{tabular}
\end{table}

\begin{figure}[h]\centering
\includegraphics[width=0.62\textwidth]{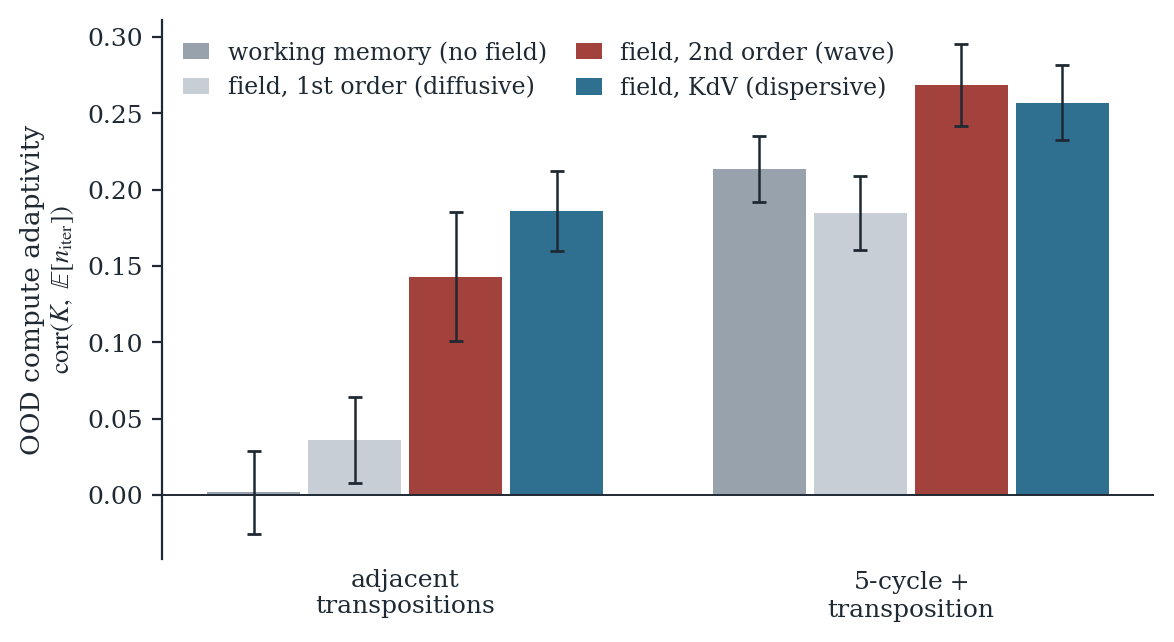}
\caption{OOD compute adaptivity by variant and generator set ($n{=}10$,
$\pm$SEM), protocol v3 (learnable physics). The \texttt{hbp\_first} arm runs at
its original caps: the twenty-seed replication of Sec.~\ref{sec:v4} shows that
the 5-cycle gap displayed here is confounded with capacity and vanishes once
the caps are equalized.}
\label{fig:ood}
\end{figure}

Three observations. (1)~\emph{A second-order signal, later bounded}: the
diffusive limit ties with the control in both generator sets, and the direct
wave$-$diffusion contrast gives $t{=}1.90$ ($p{=}0.09$) and $t{=}3.49$
($p{=}0.007$). Both are exploratory (outside the pre-registered Holm-4 family)
and both run with the \texttt{hbp\_first} caps \emph{not} equalized;
Sec.~\ref{sec:v4} shows that once they are, the 5-cycle contrast disappears
($+0.014$ $[-0.013,+0.040]$, n.s.)---part of what reads here as order was
capacity.
(2)~\emph{Statistical honesty}: none of the primary contrasts individually
survives Holm-4 ($p{=}0.025/0.028$ against
$\alpha_{\mathrm{Holm}}{=}0.0125/0.0167$). We report no combined $p$: the two
generator sets share the same ten seeds, so a Fisher combination would violate
independence. We present the effect as directionally consistent evidence from a
single confirmatory batch---not as definitive, and not as an independent
replication, since v2 and v3 re-run the same cells and seeds.
(3)~The compute slope in pure OOD
($\mathbb{E}[n_{\mathrm{iter}}|K{\in}18..24]-\mathbb{E}[n_{\mathrm{iter}}|K{\in}12..15]$)
is maximal for the second order in both sets ($0.45$ and $2.45$), though in
5-cycle the \texttt{gating\_wm} control is a near-tie ($2.29$) and
\texttt{hbp\_first} ($1.49$) falls \emph{below} the control, so the ordering is
monotone only in \texttt{adjacent}. OOD accuracy
does not separate across variants ($\approx0.06$--$0.08$; the task does not
extrapolate for anyone): the effect is on the \emph{allocation} of computation.

\begin{figure}[t]\centering
\includegraphics[width=\textwidth]{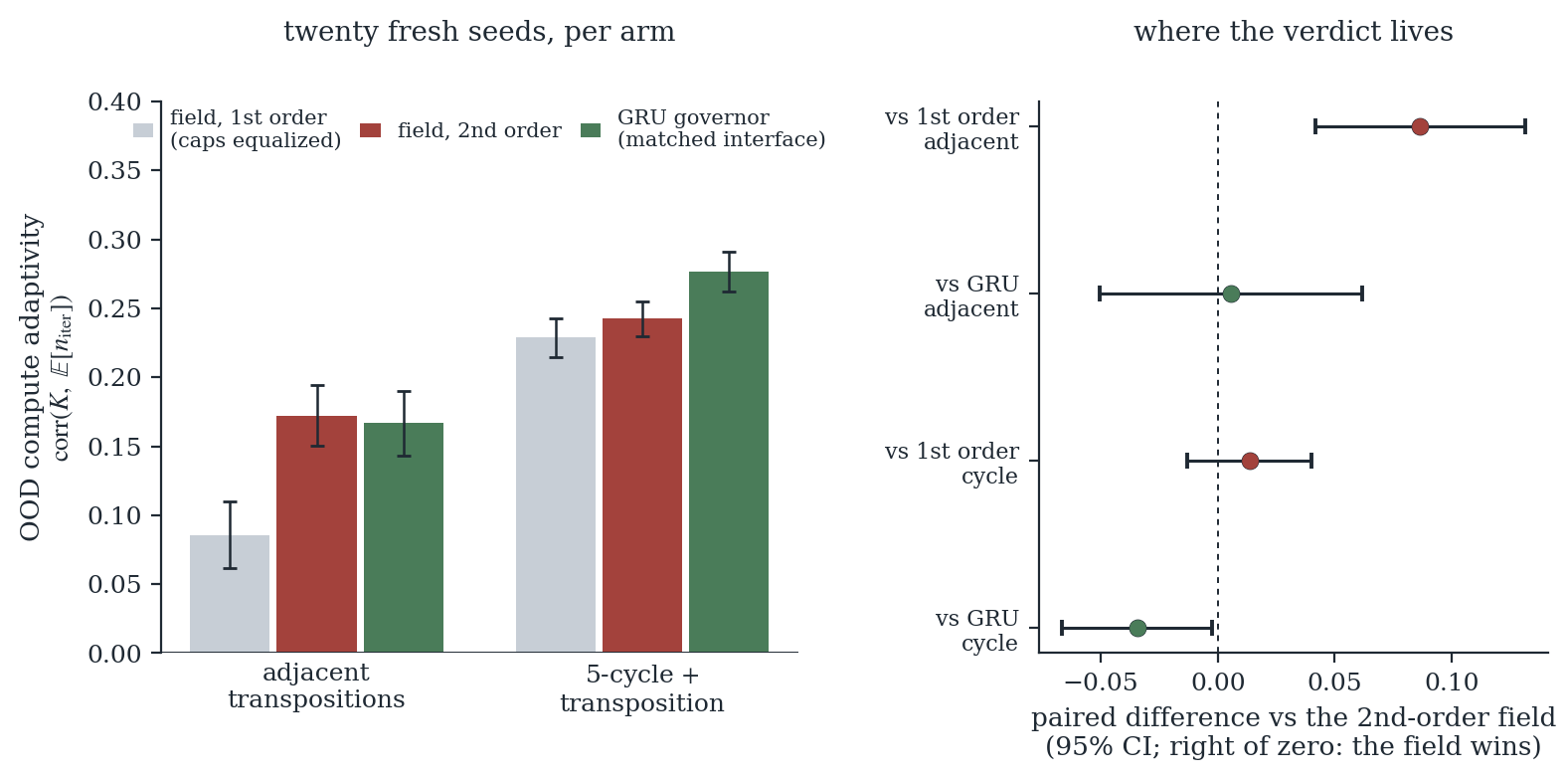}
\caption{The deconfounding campaign of this section (twenty fresh seeds,
$20$--$39$). \emph{Left:} out-of-distribution compute adaptivity per arm and
generator family; the first-order arm now runs at caps equalized to the
second-order ones, which the unconditionally contractive implicit kernel of
Proposition~\ref{prop:imex} makes safe. \emph{Right:} the same evidence as
seed-paired differences from the second-order field, which is where the
verdict actually lives. The order effect clears zero only in
\texttt{adjacent}; the matched-interface GRU governor ties the field there and
its interval falls entirely below zero in \texttt{cycle\_transp}. Error bars
are $\pm$SEM (left) and $95\%$ CI (right).}
\label{fig:v4}
\end{figure}

\subsection{Deconfounding order, and a matched-interface learned governor}
\label{sec:v4}
At review's request we ran a preregistered deconfounding campaign
(twenty fresh seeds 20--39, never used before; OOD protocol unchanged;
paired by seed, one-sided, Holm-2 per hypothesis). The first-order arm
was rebuilt with caps \emph{equalized} to the second-order ones
($c_{\max}{=}0.7$, $\omega_0$ up to $1.8$). The equalization is of the
coefficient \emph{caps} only: the second-order branch still carries a velocity
state, hence twice the effective controller state, and we did not run the two
further controls that would separate order from state size and from
underdamping (a first-order arm with $2d_h$ state; an overdamped second-order
arm at the same caps); the \texttt{adjacent} effect below must be read with
that residual confound. The equalized caps are what the unconditionally
contractive implicit kernel of Proposition~\ref{prop:imex} makes safe:
the original \texttt{hbp\_first} caps were an artifact of the explicit
$\zeta$-coupled integrator, so the earlier order contrast was confounded
with capacity. The second new arm replaces the physical integrator with
a per-node GRU cell of identical interface---same interoception, same
external forcing, same modulation heads.

The results execute the preregistration's honest branches. For reference, mean
pure-OOD adaptivity per arm ($n{=}20$, $\pm1$ SE) is \texttt{hbp\_full}
$0.172\pm0.022$ / \texttt{hbp\_first\_eq} $0.086\pm0.024$ / \texttt{hbp\_gru}
$0.167\pm0.023$ in \texttt{adjacent}, and $0.242\pm0.013$ / $0.229\pm0.014$ /
$0.277\pm0.015$ in \texttt{cycle\_transp}---the GRU arm in
\texttt{cycle\_transp} is the best single arm of the study. The order
effect is strong in \texttt{adjacent}
($\Delta_{\mathrm{corr}}=+0.087$ $[+0.042,+0.132]$, $t{=}4.0$,
$p{=}3.5\times10^{-4}$) and not detected in \texttt{cycle\_transp} once capacity
is equalized ($+0.014$ $[-0.013,+0.040]$, n.s.): under Holm the order hypothesis is not
confirmed as a family-general claim, and we report it as
generator-bounded---part of the original contrast was capacity. The GRU
governor is statistically indistinguishable from the field in \texttt{adjacent}
($+0.006$ $[-0.051,+0.062]$, n.s.---an interval too wide to certify
equivalence) and
nominally exceeds it in \texttt{cycle\_transp} ($-0.035$
$[-0.067,-0.002]$), with mutual accuracy non-inferiority everywhere
(margin $0.02$). We therefore downgrade \emph{structure yes} to:
\emph{a stateful recurrent governor with this interoceptive interface
suffices}. The field's distinctive contribution is that it is \emph{an}
implementation whose stability is certifiable
(Appendices~\ref{app:scope}--\ref{app:proofs}); the GRU carries no such
certificate.

\subsection{The zoo of physics: KdV dispersion neither helps nor hurts}
On the same protocol ($n{=}10$), \texttt{hbp\_kdv} (wave $+$ dispersion
$\beta\mathbf{A}^3$ $+$ saturated nonlinearity, gyroscopic placement of
Lemma~\ref{lem:placement}) matches the pure wave in OOD adaptivity:
$\mathrm{corr}_{\mathrm{OOD}}=0.186$ vs.\ $0.143$ (adjacent) and $0.257$ vs.\
$0.268$ (5-cycle); the paired \emph{kdv$-$wave} contrast is null in both
($t{=}1.12$, $p{=}0.29$; $t{=}{-}0.59$, $p{=}0.57$). Since \texttt{hbp\_kdv} is
second order, it separates from the
\texttt{gating\_wm} control in \texttt{adjacent} ($t{=}7.05$, $p{<}0.001$) but
not in 5-cycle ($t{=}2.08$, $p{=}0.07$, n.s.); both contrasts are exploratory,
outside the pre-registered Holm-4 family. That is: what separates
\texttt{hbp\_kdv} from \texttt{hbp\_full} ---the third-order dispersive
operator--- shows \emph{no detectable effect} at $n{=}10$ in either direction,
while what they share ---the second-order inertia--- is what covaries with
adaptivity, and only in \texttt{adjacent} once Sec.~\ref{sec:v4} bounds it. The zoo runs with the
corrected dynamics (gyroscopic placement, IMEX with antisymmetric kernel,
complex Schur--Cohn certificate active; the exact spectral radius was checked
$<1$ at initialization and at the $\beta$ cap, not at every trained
operating points). At $N{=}6$ and $\sim$5--11 ticks there is no solitonic regime
to exploit; the result bounds the role of dispersion in small module graphs and
reinforces the reading we can defend: within this zoo, what covaries with
adaptivity is the second-order inertial term rather than the richness of the
spatial operator---a reading that Sec.~\ref{sec:v4} then confines to
\texttt{adjacent} and shows is not exclusive to a physical integrator.

% =====================================================================
\section{A mechanism null: the field's physics does not move accuracy}
\label{sec:null}

The family lets us ask directly: does the physical \emph{regime} of the field
matter for performance? Three independent lines of evidence, subjected to
adversarial verification (refutation panels with pre-specified confounder
controls), answer no:

\paragraph{(1) Forced physics.} Training with $\alpha$ \emph{imposed} (pure wave
vs.\ pure diffusion) over four tasks (composition in $S_5$, functional
iteration, modular sum with distractors, recall), accuracy is identical in all
($|\Delta|\le0.012$, within binomial error), with both regimes certified stable.

\paragraph{(2) Gated physics.} Letting interoception choose $\alpha,D,b$ per tick
(\texttt{hbp\_mix}), the gating converges to a \emph{constant} physics
($\alpha\to0.79$, deviation across inputs $\sim10^{-5}$, also per dimension:
$\alpha\in[0.76,0.84]$ across the $N{\times}d_h$ components), with the head
weights \emph{growing} under weight decay (the gradient arrives; the loss
\emph{prefers} constant physics) and with a strongly input-sensitive
initialization retrained back to the same attractor. Ablating $D$ or $b$ changes
accuracy by $<0.008$.

\paragraph{(3) Dual-regime tasks.} We design tasks with two signaled modes whose
temporal modulation demand differs (oscillatory retention vs.\ monotone
settling), mounted on composition in $S_5$ to guarantee iteration (the reasoner
reaches $\mathbb{E}[n_{\mathrm{iter}}]$ up to $11$, increasing with $K$). Even so,
wave and diffusion perform identically \emph{within each mode}
($|\Delta|\le0.005$), which bounds to $\le0.005$ the gap a physics-switching
oracle could have achieved \emph{on these tasks}; we do not claim the bound
transfers to demands we did not construct.

\paragraph{(4) Beyond \emph{local} physics: non-locality and a flow substrate.}
The three lines above vary the \emph{type} of local physics. We further test two
\emph{structural} extensions that the set-point argument does not trivially
dissolve. \emph{(a) Non-locality.} We couple the modulation to a stream function
enslaved by Poisson on the graph, $\psi=\mathbf{L}^{+}(h-h^\ast)$, so that the
modulation of each node depends on \emph{the whole} field (the discrete analog of
the non-local dipolar field in NMR, $\nabla^2\psi=-\omega$). On the same OOD
protocol, non-locality does not improve accuracy ($\Delta\approx0$) and
\emph{degrades} compute adaptivity in 5-cycle$+$transp.
($\Delta z={-}0.069$, $t={-}3.88$) while showing no effect in \texttt{adjacent}
($\Delta z={-}0.023$, $t={-}0.66$, n.s.). A plausible reading---which this
exploratory contrast does not establish---is that the difficulty signal
governing \emph{halting} is local to the reasoner node, so globalizing it
blurs it. \emph{(b) 2D flow substrate.} We
carry the field to a sheet with incompressible Navier--Stokes dynamics
(vorticity $\omega$, $\psi=\mathbf{L}^{+}\omega$, velocity
$u=\partial_y\psi,\,v=-\partial_x\psi$, advected passive scalar $S$), where the
field \emph{computes} transport instead of modulating. A \emph{fundamental}
physical limit closes it: by incompressibility ($\nabla\!\cdot u=0$,
conservation of material area) a flow can transport and \emph{mix} but never
\emph{concentrate} information at a point ---concentrating requires negative
divergence, i.e.\ compressibility---; in a delivery task with local readout, the
mass delivered to the destination is $\approx0$ and the temporal dynamics makes
it worse (mixes more). The incompressible flow substrate is a mixer, not a
delivering computer. Both extensions confirm the null by routes that are not
``another type of local physics''.

\paragraph{Interpretation: temporal \emph{coarse-graining}.} The reasoner does
not process one token per tick: it groups ($\sim$5 ticks for $\sim$19
operations), and gradient descent systematically finds solutions in which the
modulation demand is a quasi-static \emph{set-point} per instance. Both branches
of the family reach any forced set-point at equilibrium; they differ only in the
transient ($\sim$4 ticks), which is invisible to the loss. Physics switching
would only pay off with oscillatory demands aligned tick by tick, which the
architecture does not force and the gradient does not discover. Together with
(4), the pattern is broad: neither the types of local physics we instantiated, nor
non-locality, nor the substrate-computer role change the result. This null is, we
believe, informative beyond our model: it suggests that dynamic modulatory
substrates in adaptive architectures are exercised by their \emph{recurrent
structure} (statefulness, and---in one of our two generator families---the
order of the dynamics) rather than by their \emph{type} of physics, and
that an incompressible field, despite its elegance, has a structural ceiling as a
spatial computer.

\paragraph{The field as an evidence accumulator: the niche is occupied.} One
employment survives the set-point argument on paper: a damped second-order
field is, formally, a leaky integrator with inertia---the physical form of
sequential evidence accumulation. If the reasoner's per-tick posterior stream
(halting mass, readout margin and entropy, prediction flips, trajectory
speed) were noisy in a way that a temporal filter could exploit, the field
would finally carry load. We tested the premise with a kill-gate whose primary
cell and pass threshold ($\Delta\mathrm{AUC}\ge0.03$) were fixed in a dated
design document before the verdict---not a formal pre-registration, and the
probe was recalibrated after a first null (below)---run
on the twelve frozen solvers of the integration program of the
companion paper \compcite{} (six seeds $\times$ two independent runs, a reasoner-with-value
substrate distinct from the $n{=}10$ protocol above; we report the gate
here because it closes this paper's thesis, and the transfer of its
conclusion rests on the shared recurrent-reasoner structure), using a
probe whose
sensitivity we certified with planted distributed signals (load $0$ gives
$-0.001\pm0.001$, unbiased; load $0.20$ gives $+0.026$ $[+0.020,+0.031]$, i.e.\
sensitivity at the scale of---but not above---the $0.03$ threshold; an
uncalibrated version of the
same probe paid an overfitting tax of ${\sim}0.018$ AUC that biased it
toward the null, and the positive control caught it before any verdict). The
result is flat: the full stream predicts success no better than the last
tick alone, $\Delta\mathrm{AUC}=+0.0007$ $[-0.0065,\,+0.0079]$ in the
predeclared primary cell---against a pass threshold of $0.03$, which it does
not meet---with every secondary cell within
$[-0.0012,+0.0073]$; and the detrended stream carries almost no temporal
structure (lag-$2$--$6$ autocorrelation $\approx+0.05$; non-DC spectral
peaks in $0\%$ of trajectory-speed traces and $17\%$ of margin traces,
none exploitable by the certified probe). The mechanism is
structural rather than statistical: the recurrent state already integrates
its own history, so the last tick is a sufficient statistic of the
trajectory and the accumulator's niche is occupied by construction. This
closes the last route by which the field could have carried content, and
sharpens the paper's thesis into an audit: of the employments we were able to
imagine and test---source of computation, load-bearing physics, evidence
accumulator---none is performed better by the field than by a dedicated
component of the architecture, except the one the field already performs. It
modulates compute; that is the job, and it is a real one---but not an
exclusive one: a matched learned recurrence modulates compute at least as well
(Sec.~\ref{sec:v4}), so what stays the field's own is the certificate, not the
job.

% =====================================================================
\section{Discussion and limitations}

\subsection{What the homeostatic field does and does not do}
The results tease apart three questions. \emph{Does it improve accuracy?} No: the
recurrence explains in-distribution accuracy (the working memory adds nothing
on top of it in our own table), and none of the members of the family we
instantiated ---forced, gated or from the zoo--- moves it. \emph{Does the type of
physics matter?} No for performance (Sec.~\ref{sec:null}); the diffusion equation
plays here the role of the contrast that gives meaning to the central claim, and
the KdV operators bound the role of dispersion at small $N$. \emph{Does structure
matter?} Partly: the second-order dynamics produces a compute allocation more
robust out of distribution than the control's learned policy, and the effect is
independent of the tuning of its constants (BF16-bug A/B) ---but under
deconfounding at equalized capacity it holds in one generator family and not the
other, and a matched-interface GRU reaches the same place---nominally going
past it in the family where the order effect vanishes (Sec.~\ref{sec:v4}).
The field is a \emph{compute controller}, not an accuracy enhancer; inertia is an
active ingredient in one family, and the class of viable governors is wider than
the field.

\subsection{Where it fits: the field as governor, not as cortex}
The pattern of results is not an ambiguous failure but a \emph{location signal}.
That a dynamic internal field is neutral for accuracy yet load-bearing only
for compute allocation---in one generator family, and not exclusively
(Sec.~\ref{sec:v4})---places it, on the evidence we have, in the \emph{metacognitive control /
resource governor} layer of a cognitive architecture ---and not in the reasoning
one. In biological terms, it is not the cortex (which computes the answer) but
the \emph{neuromodulatory and homeostatic} substrate (which sets gain, arousal
and budget); in systems terms, it is the \emph{scheduler/governor}, not the
arithmetic unit; in control terms, the \emph{external} supervisory loop, slow and
certified, wrapped around the fast inference loop. Its universal
interoception/modulation interface (Fig.~\ref{fig:concept}) is precisely that of
a governor: it observes the computational ``vital signs'' (progress, uncertainty,
effort) and modulates \emph{how much} and \emph{where} to compute, without
deciding \emph{what}. The certified stability (Sec.~\ref{sec:theory}) ceases to
be a technical detail and becomes the \emph{design} property a governor needs: a
control loop that provably does not diverge. It is also, after the GRU
comparison, the property that survives as distinctive: the field is not the
strongest governor of its class, it is the one that comes with a proof. The
substance null and the bounded structure effect, together, do not disqualify the
field ---they relocate it: it is a viable and certifiable \emph{compute
governor}, not a cognition module.

\subsection{What this component is not, and what is missing}
Our thesis is bounded: the HBP is \emph{one} piece ---the regulatory one--- of a
broader intelligent system, not the system. It does not reason, it does not plan
and it does not replace the recurrent reasoner, whose recurrence is what
explains in-distribution accuracy. A complete architecture pursuing cognition would
require, in addition to the governor characterized here, a deliberative module
(the ``cortex''), a world model/memory, and a system of drives and goals; all of
them lie outside this work and are the subject of future development. This
article's contribution is to close, with rigor and honesty, the question about
\emph{that} piece: a dynamic internal field can govern an LLM's computation in a
certifiable way, and what is distinctively its own is the certificate---not
capability, and not even the second-order structure, which is bounded to one
generator family and matched by a learned recurrence.

\subsection{Limitations}
\emph{(1)}~A single task family ($S_5$), though theoretically sharp.
\emph{(2)}~The finding's metric is compute allocation, not OOD accuracy (which
does not separate), and the structural effect that remains is bounded to one
of the two generator sets and is reproduced by a matched-interface GRU
governor (Sec.~\ref{sec:v4}). \emph{(3)}~Small scale ($4.2$--$5.6$M
parameters; $n{=}3$ in-dist and $n{=}10$ OOD in v3, raised to $n{=}20$ in the
deconfounding campaign) and v3 primary contrasts that do not survive Holm-4
individually; the deconfounding equalizes coefficient caps but not controller
state size. \emph{(4)}~$N{=}6$ nodes: no solitonic regime for the dispersion;
the zoo of physics at large $N$ remains open. \emph{(5)}~The mechanism null is a
result about what the gradient \emph{finds}, not about what the architecture
\emph{could} express: a task with an architecturally forced tick-aligned
oscillatory demand could reverse it. \emph{(6)}~The flow-substrate limit
(Sec.~\ref{sec:null}, line 4) is specific to \emph{incompressibility}: a field
with learnable sinks (compressible advection--diffusion--reaction) could deliver,
at the cost of abandoning the vorticity--stream-function formulation and its
elegance; we did not evaluate it.

\subsection{Future work}
(i)~Larger module graphs (nodes per layer/head) where dispersion and advection
have sufficient spectrum; (ii)~Cayley discretization of the gyroscopic term
(exact discrete neutrality, Prop.~\ref{prop:verlet}); (iii)~tasks where the
compute budget is the effective bottleneck, to turn allocation robustness into
accuracy; (iv)~interoceptive signals discriminative of the problem (the current
ones are nearly input-invariant at fixed $K$, which limits any gating);
(v)~integration with NMR applications, where the formal analogy is transferable.

% =====================================================================
\section{Conclusion}
We asked whether a dynamic internal field can modulate an LLM's cognition. The
answer, supported by proved results with their scope stated
(App.~\ref{app:scope}), numerical audits of the implementation, and two
pre-declared empirical campaigns (the v3 protocol and the v4 deconfounding),
is
and threefold: \textbf{substance no, structure only in part, certifiability
yes}. A second-order homeostatic field is a viable and certifiable \emph{compute
governor} ---it confers on compute allocation an out-of-distribution robustness
that the \texttt{gating\_wm} control's learned halting policy does not have, in
one of two generator families, while a matched-interface GRU governor reaches
the same place---and nominally goes past it in the other family---without a
certificate---, but its \emph{type} of physics is irrelevant for accuracy in
every configuration we tested: we propose that the gradient laminates temporal
demand to set-points, and the
null is \emph{broad}, resisting the sweep of local physics, Poisson-type
non-locality and a 2D Navier--Stokes flow substrate (closed by an
incompressibility limit). The field modulates, it does not think. We thus place
the HBP in the \emph{metacognitive control} layer of a cognitive architecture
---its computational ``brainstem''--- and provide the stability tools
(gyroscopic-placement lemma, complex Schur--Cohn certificates, unconditional
IMEX scheme) that a governor of this kind needs ---and that its learned
competitors do not have. We believe the substance/structure/certifiability
demarcation, the broad null with a mechanism, and the architectural location
that follows from it are useful to anyone designing the layer that governs
---rather than performs--- a model's computation.
Characterizing \emph{one} piece well, and its limits, is the necessary step
before assembling the rest.

{\sloppy \paragraph{Reproducibility.} All the code is pure PyTorch, with no third-party
\emph{framework} dependencies, with fixed seeds and protocols frozen before
observing the data. The field's implementation and its certificates are in
\texttt{model/hbp.py} (PDE family, gyroscopic placement, IMEX solver,
\texttt{stability\_penalty} and the exact spectral radius
\texttt{certificate\_spectral\_radius}); the 2D flow substrate in
\texttt{model/flow2d.py}. Reproducing the tables and figures: the pre-registered
benchmark and its paired analysis (\texttt{experiments/benchmark\_v3.py},
\texttt{experiments/benchmark\_report\_v2.py}), the v4 deconfounding campaign
(\texttt{PREREG\_V4.md}, \texttt{experiments/benchmark\_v4.py},
\texttt{results\_benchmark\_v4/report\_v4.json}), the accumulator kill-gate
(\texttt{mhbp/tasks/reasoner\_g0/n4\_g1.py}), the KdV zoo (\texttt{experiments/benchmark\_zoo.py}), the
non-locality A/B (\texttt{experiments/benchmark\_elliptic.py}), and the mechanism-null studies
(\texttt{experiments/\_alpha\_scan.py}, \texttt{experiments/\_pde\_study.py}, \texttt{experiments/\_twin\_pilot.py},
\texttt{experiments/\_gates\_flowroute.py}). A regression test verifies that the field is
actually \emph{wired in} ---ablating it changes the forward pass and the
compute policy, though not accuracy--- (\texttt{experiments/\_audit\_hbp.py}), a check of
plumbing rather than of the functional load this paper's null denies; and physical verifications of the
2D solver in \texttt{experiments/\_check\_flow2d.py}. The physical parameters are anchored to
FP32 (\texttt{pin\_fp32}); we document the BF16 precision bug and its A/B as part
of the experimental record.

\appendix
\section{Scope and status of the stability results}\label{app:scope}
At a reviewer's request we state precisely what is proven, under which
hypotheses, and what is verified numerically.

\paragraph{Proven.} The discrete Verlet criterion of
Prop.~\ref{prop:verlet} is proved in App.~\ref{app:proofs}: it is necessary
and sufficient per latent root, and the reduction uses no commutation
hypothesis. The passage from latent roots to an operator-level certificate via
the parameter box is sufficient only, and Remark~\ref{rem:caja} both proves the
correct finite check and reports how conservative it is.
The placement lemma's structural claims---that
antisymmetric coupling enters the second-order branch gyroscopically and
the first-order branch positionally, and that the two placements have
qualitatively different stability consequences---follow from the
symmetric/antisymmetric decomposition given in the stability section,
with the standard energy arguments sketched there ---standard because they
are classical: (a) is Kelvin--Tait--Chetaev and (b) is the accretive/numerical
range argument. Part (c) of the dichotomy is \emph{not} proved here: showing
that one Lyapunov candidate stops decreasing does not establish instability,
and the corresponding instability theorems are in
\citet{bulatovic1999stability}. The unconditional contraction of the IMEX
kernel does \emph{not} follow from Schur--Cohn applied mode by mode ---that
route needs simultaneous diagonalization, which fails here--- but from the
coercivity of the symmetric part, which is precisely what buys the absence of
a commutation hypothesis (App.~\ref{app:proofs}).

\paragraph{Exact under hypothesis, and the hypothesis is never met.} The
closed-form flutter threshold $\beta\rho(A^3)<2\zeta\omega_0^2$ is exact when
\emph{both} the stiffness and the damping operator are multiples of the
identity: $K=\omega_0^2 I$ (i.e.\ $c=0$) \emph{and} $C=2\zeta\omega_0 I$
(i.e.\ no structural diffusion). The second is a hypothesis the earlier
version of this appendix left silent. Two corrections follow. First, the
parenthetical equivalence ``$c=0$, or $[\mathbf{L},A^3]=0$'' was wrong: under
commutation with $c>0$ the system does decouple, but the per-mode threshold
becomes $\beta|\nu_j|<2\zeta\omega_0\sqrt{\omega_0^2+c^2\ell_j}$, so the
published threshold is then sufficient but \emph{not} necessary, hence not
exact. Second, and more consequentially: in the implementation $c$ is
parameterized as $c_{\max}\sigma(\cdot)$ and is therefore \emph{strictly
positive in every run}, so the exact regime covers none of the systems we
report. Worse for the reader who checks, the headline arm instantiates no
antisymmetric operator at all ($\mathbf{G}=0$ by configuration), so the
placement-and-flutter apparatus is a \emph{design-exclusion result about a
variant we chose not to build}, not a certificate of anything we ran. We use
the threshold as the design guide it is, and claim nothing more.

\paragraph{Verified numerically.} The identities and thresholds of the
stability section were validated against direct numerical experiments on
the discrete system (residuals at machine precision for the identities;
threshold location within the stated tolerance; no instability observed
under the certified conditions in long-horizon integration). These are
audits of the \emph{implementation} at fixed parameter draws, not
theorems.

\paragraph{On trained runs.} The operative guarantee during and after
training is the runtime certificate: the exact spectral radius of the
one-step operator, computed by probing the implementation
(\texttt{certificate\_spectral\_radius}), together with the training-time
stability penalty. Certificates are audited on trained checkpoints in
dedicated scripts; they are not logged at every training step, and we
scope the paper's verification claims accordingly.

\paragraph{Evidence-accumulator gate.} The kill-gate of
Section~6 ran on the twelve frozen solvers of the integration program of
the companion paper \compcite{}---a reasoner-with-value
substrate distinct from the
$n{=}10$ protocol of this paper; the transfer of its conclusion rests on
the shared recurrent-reasoner structure, as stated in the main text.

\section{Proofs of the certified results}\label{app:proofs}
We give complete proofs of the three results whose scope
Appendix~\ref{app:scope} delimits: the placement dichotomy, the exact
flutter threshold under the stated hypothesis, and the unconditional
IMEX kernel bound. Throughout, $K$ is symmetric with
$K\succeq\omega_0^2 I$ (in the field, $K=\omega_0^2 I + c^2\mathbf{L}$
with $\mathbf{L}$ the graph Laplacian, positive semidefinite), and $G$
is real antisymmetric ($G^\top=-G$).

\begin{lemma}[Placement dichotomy]\label{lem:placement-proof}
Let $F(t)$ be a bounded forcing. (i) In the first-order branch
$\dot u = -Ku - Gu + F$, the antisymmetric term is norm-neutral:
$\tfrac{d}{dt}\tfrac12\|u\|^2 = -u^\top K u + u^\top F$, so stability is
governed by $K$ alone. (ii) In the second-order branch with
\emph{gyroscopic} placement, $\ddot x + 2\zeta\omega_0\dot x + Kx +
G\dot x = F$, the mechanical energy
$E=\tfrac12\|\dot x\|^2+\tfrac12 x^\top Kx$ satisfies
$\dot E = -2\zeta\omega_0\|\dot x\|^2 + \dot x^\top F$: the antisymmetric
term does no work and $E$ remains a Lyapunov function. (iii) With
\emph{circulatory} placement, $\ddot x + 2\zeta\omega_0\dot x + Kx + Gx
= F$, one gets $\dot E = -2\zeta\omega_0\|\dot x\|^2 - \dot x^\top Gx +
\dot x^\top F$; the cross term $\dot x^\top Gx$ is sign-indefinite and
can pump energy, so stability is conditional (Lemma~\ref{lem:flutter}).
\end{lemma}

\begin{proof}
All three identities follow by differentiating the stated functional
along trajectories and using $v^\top G v=0$ for every real $v$ (by
antisymmetry, $v^\top Gv = (v^\top Gv)^\top = -v^\top Gv$). In (i),
$u^\top\dot u = -u^\top Ku - u^\top Gu + u^\top F$ and the middle term
vanishes. In (ii), $\dot E = \dot x^\top\ddot x + \dot x^\top Kx$;
substituting $\ddot x$ and cancelling $\dot x^\top G\dot x = 0$ gives
the claim. In (iii) the same substitution leaves $-\dot x^\top Gx$,
which has no sign.
\end{proof}

\begin{lemma}[Exact flutter threshold for $K=\omega_0^2 I$]
\label{lem:flutter}
Consider the unforced circulatory system
$\ddot x + 2\zeta\omega_0\dot x + \omega_0^2 x + \beta A^3 x = 0$ with
$\zeta,\omega_0,\beta>0$ and $A$ real antisymmetric. The origin is
asymptotically stable if and only if
$\beta\,\rho(A^3) < 2\zeta\omega_0^2$,
where $\rho(\cdot)$ is the spectral radius.
\end{lemma}

\begin{proof}
$A^3$ is antisymmetric, hence normal with purely imaginary spectrum
$\{i\nu_j\}$, $\nu_j\in\mathbb{R}$, and $\mathbb{C}^N$ decomposes into
orthogonal invariant eigenplanes. Because the remaining operators are
multiples of the identity, the system block-diagonalizes over these
planes, and on the mode with eigenvalue $i\nu$ the characteristic
polynomial is
$p(\lambda)=\lambda^2 + 2\zeta\omega_0\lambda + \omega_0^2 + i\beta\nu$.
At $\beta=0$ both roots lie in the open left half-plane (real part
$-\zeta\omega_0<0$ for $\zeta<1$; two negative real roots
$-\zeta\omega_0\pm\omega_0\sqrt{\zeta^2-1}$ for $\zeta\ge1$). Roots depend
continuously on $\beta$, so instability can only arise by a root
crossing the imaginary axis. Setting $\lambda=i\omega$
($\omega\in\mathbb{R}$) and separating real and imaginary parts:
$-\omega^2+\omega_0^2=0$ and $2\zeta\omega_0\omega + \beta\nu = 0$,
i.e.\ $\omega=\pm\omega_0$ and $\beta\nu = \mp 2\zeta\omega_0^2$. A
boundary crossing therefore occurs exactly at
$\beta|\nu| = 2\zeta\omega_0^2$; below it no root can have reached the
axis, above it (checking the sign of
$\partial\,\mathrm{Re}\,\lambda/\partial\beta$ at the crossing, which is
positive) one root has positive real part. Taking the worst mode
$|\nu|=\rho(A^3)$ gives the claim. For $K=\omega_0^2I+c^2\mathbf{L}$
with $c>0$ the eigenplanes of $A^3$ are no longer invariant unless
$[\mathbf{L},A^3]=0$, and the threshold is not claimed exact
(Appendix~\ref{app:scope}).
\end{proof}

\begin{proposition}[Unconditional IMEX kernel bound; standard, recalled for
completeness]\label{prop:imex-proof}
Let $M = I + \lambda(K+G)$ with $\lambda>0$, $K$ symmetric with
$K\succeq\omega_0^2 I$, and $G$ real antisymmetric. Then $M$ is
invertible and $\|M^{-1}\|_2 \le (1+\lambda\omega_0^2)^{-1} < 1$,
with no condition on $\lambda$ (hence on $dt$, $\zeta$) and no
commutation hypothesis. Consequently the backward-Euler update
$u' = M^{-1}(u+\lambda F)$ is a strict contraction toward the forced
equilibrium whenever $F$ is \emph{independent of the state}. When $F$
depends on $u$ ---as it does in the implementation, where it collects the
saturating perturbation and the learned drives--- the composed map has
Lipschitz constant at most $(1+\lambda L_F)/(1+\lambda\omega_0^2)$, which is
below one only if $L_F<\omega_0^2$. That is a condition on the operating box,
not a consequence of this proposition, and we do not assert it: at the low
end of the admissible range ($\omega_0^2$ small against the drive gains) it
need not hold.
\end{proposition}

\begin{proof}
For any $x\in\mathbb{C}^N$, $x^*Gx$ is purely imaginary: since $G$ is
real antisymmetric, $(x^*Gx)^* = x^*G^\top x = -x^*Gx$. Hence
$\mathrm{Re}\,x^*Mx = \|x\|^2 + \lambda\,x^*Kx \ge
(1+\lambda\omega_0^2)\|x\|^2$, using $K\succeq\omega_0^2 I$. By
Cauchy--Schwarz, $\|Mx\|\,\|x\| \ge |x^*Mx| \ge \mathrm{Re}\,x^*Mx$, so
$\|Mx\| \ge (1+\lambda\omega_0^2)\|x\|$ for all $x$; thus $M$ is
injective (hence invertible) and $\sigma_{\min}(M)\ge
1+\lambda\omega_0^2$, which is the stated bound on $\|M^{-1}\|_2$.
\end{proof}

\begin{remark}
Proposition~\ref{prop:imex-proof} is the reason the first-order branch can run
with equalized caps ($c_{\max}$, $\omega_{0,\max}$ at the second-order
values) in the v4 deconfounding study: the implicit kernel is
unconditionally contractive, so the caps that the explicit
$\zeta$-coupled Euler forced on \texttt{hbp\_first} are not needed.
When the antisymmetric coefficients are \emph{gated} per instance, $G$
cannot be folded into a fixed kernel and the certificate falls back to
the runtime spectral-radius probe, as stated in
Appendix~\ref{app:scope}.
\end{remark}

\paragraph{Author contributions (CRediT).} \ifanon Withheld for double-blind review. \else
\textbf{F.M.A.-C.} (corresponding): conceptualization, methodology, software,
formal analysis, investigation, visualization, project administration,
writing---original draft. \ 
\textbf{F.G.M.}: conceptualization, methodology, formal analysis, validation,
supervision, writing---review and editing. \ 
\textbf{A.A.}: software, data curation, resources, investigation, validation,
writing---review and editing. \ 
\textbf{I.F.}: conceptualization, supervision, funding acquisition, resources,
validation, writing---review and editing. \fi

\paragraph{Data and code availability.} The full program---model code,
preregistrations, run artifacts and the scripts that produce every table and
figure in this paper---is covered by the MIT licence.
\ifanon The repository is withheld here for double-blind review; its URL and
an archival DOI will be given in the camera-ready version.
\else\ifdefined\coderepo It is available at \coderepo.
\else It is available from the corresponding author, and a public repository
with an archival DOI will accompany the published version.\fi\fi
Every number in the text enters as a macro transcribed from an archived
artifact and checked against it, so each one can be traced back to the file it
came from.

\paragraph{Use of generative AI.} The experiments were designed, executed and
adjudicated by the authors. A large language model was used as a coding and
drafting assistant throughout, including for implementation, for adversarial
review of the manuscript's internal consistency, and for literature search;
all claims, numbers and verdicts reported here were verified by the authors
against the archived artifacts.

\begin{proposition}[Discrete criterion for the wave branch; Prop.~\ref{prop:verlet}]
\label{prop:verlet-proof}
Let $\mathbf{K}=\mathbf{K}^\top$, $\mathbf{C}=\mathbf{C}^\top$ be real and
$\mathbf{G}=-\mathbf{G}^\top$ real, and consider the position-Verlet step with
velocity coupling
\[
u_{t+1}=\big(2\mathbf{I}-\Delta t^2\mathbf{K}-\Delta t(\mathbf{C}+\mathbf{G})\big)u_t
-\big(\mathbf{I}-\Delta t(\mathbf{C}+\mathbf{G})\big)u_{t-1}.
\]
Let $z$ be an eigenvalue of its companion operator and $x\neq0$ the
corresponding latent vector, normalized to $\lVert x\rVert=1$. Put
$q=\Delta t^2\,x^*\mathbf{K}x$, $g=\Delta t\,x^*\mathbf{C}x$ and
$\tilde\mu=\Delta t\,\mathrm{Im}(x^*\mathbf{G}x)$, all real. Then
$z^2+a_1z+a_0=0$ with $a_0=(1-g)-i\tilde\mu$ and
$a_1=-(2-q-g)+i\tilde\mu$, and, provided $q>0$,
\[
|z|<1\quad\Longleftrightarrow\quad
\tilde\mu^2<g(2-g)\ \ \text{and}\ \
q(g^2+\tilde\mu^2)<2g\big(g(2-g)-\tilde\mu^2\big).
\]
No commutation between $\mathbf{K}$, $\mathbf{C}$ and $\mathbf{G}$ is assumed.
\end{proposition}

\begin{proof}
\emph{Reduction.} Writing the recursion as a quadratic matrix polynomial,
$z$ is a latent root of
$P(z)=z^2\mathbf{I}-z\big(2\mathbf{I}-\Delta t^2\mathbf{K}-\Delta t(\mathbf{C}+\mathbf{G})\big)
+\big(\mathbf{I}-\Delta t(\mathbf{C}+\mathbf{G})\big)$
with $P(z)x=0$. Left-multiplying by $x^*$ and using $\lVert x\rVert=1$ gives the
scalar equation. Here $x^*\mathbf{K}x$ and $x^*\mathbf{C}x$ are real because
$\mathbf{K},\mathbf{C}$ are real symmetric, and $x^*\mathbf{G}x$ is purely
imaginary because $\mathbf{G}$ is real antisymmetric
($\overline{x^*\mathbf{G}x}=x^*\mathbf{G}^\top x=-x^*\mathbf{G}x$). Hence
$a_0,a_1$ have the stated form. Note that this step uses only the latent
vector attached to $z$; it does \emph{not} diagonalize the three operators
simultaneously, and therefore holds when they do not commute.

\emph{Criterion.} For a monic complex quadratic $z^2+a_1z+a_0$, Cohn's
criterion states that both roots lie in the open unit disk if and only if
$|a_0|<1$ and $|a_1-\bar a_1a_0|<1-|a_0|^2$. We evaluate both.

For the first, $|a_0|^2=(1-g)^2+\tilde\mu^2$, so
$|a_0|<1\iff 1-2g+g^2+\tilde\mu^2<1\iff\tilde\mu^2<g(2-g)$, which is (i).
Write $D:=1-|a_0|^2=g(2-g)-\tilde\mu^2$, positive exactly under (i).

For the second, set $s:=2-q-g$, so $a_1=-s+i\tilde\mu$ and
$\bar a_1=-s-i\tilde\mu$. Then
\[
\bar a_1a_0=(-s-i\tilde\mu)\big((1-g)-i\tilde\mu\big)
=-s(1-g)-\tilde\mu^2+i\tilde\mu\big(s-(1-g)\big),
\]
and therefore
\[
a_1-\bar a_1a_0=\big(\tilde\mu^2-sg\big)+i\tilde\mu\,(2-s-g)
=\big(qg-D\big)+i\,q\tilde\mu ,
\]
using $2-s-g=q$ and $\tilde\mu^2-sg=\tilde\mu^2-g(2-q-g)=qg-D$. Squaring,
the condition $|a_1-\bar a_1a_0|^2<D^2$ reads
\[
(qg-D)^2+q^2\tilde\mu^2<D^2
\iff q^2\big(g^2+\tilde\mu^2\big)<2qgD ,
\]
and dividing by $q>0$ gives exactly (ii). Both steps are equivalences, so the
criterion is necessary and sufficient, not merely sufficient.
\end{proof}

\begin{remark}[From latent roots to the operator: the box, and its price]
\label{rem:caja}
Certifying the whole operator requires controlling the triples
$(q,g,\tilde\mu)$, which range over the \emph{joint numerical range} of
$(\mathbf{K},\mathbf{C},\mathbf{G})$ on unit vectors. Replacing that set by
the box
$q\in\Delta t^2[\lambda_{\min}(\mathbf{K}),\lambda_{\max}(\mathbf{K})]$,
$g\in\Delta t[\lambda_{\min}(\mathbf{C}),\lambda_{\max}(\mathbf{C})]$,
$|\tilde\mu|\le\Delta t\,\rho(\mathbf{G})$ yields a \emph{sufficient}
condition, since the box contains the joint range. Verifying it over the box is
a finite computation, but not simply a check at the corners, which an earlier
version of this paper asserted without proof. In $q$ and in $\tilde\mu^2$ the
left side of (ii) increases and the right side decreases, so the worst case is
$q_{\max},|\tilde\mu|_{\max}$. In $g$ it is not monotone:
$F(g):=2g\big(g(2-g)-\tilde\mu^2\big)-q(g^2+\tilde\mu^2)$ is a cubic with
negative leading coefficient, so its minimum over an interval is attained at
an endpoint \emph{or} at the interior critical point solving
$-6g^2+(8-2q)g-2\tilde\mu^2=0$. Likewise (i) binds at whichever endpoint of the
$g$-interval is farther from $1$, since $g(2-g)$ peaks there. The certificate is
therefore the minimum of $F$ over the endpoints together with any interior
critical point. We record that this box bound is markedly conservative: over
$400$ random non-commuting draws it certified only $8$, all of them correctly.
For that reason the operative check in the experiments is the exact spectral
radius of the companion, not the box.
\end{remark}

\section{A common-Lyapunov certificate: what it buys, and where it stops}
\label{app:lmi}
The runtime probe of Sec.~\ref{sec:theory} returns a spectral radius, which
is blind to non-normal transients and says nothing about convex mixtures or
about time-varying (instance-gated) coefficients, since neither the spectral
radius nor stability under products is implied by $\rho(\Phi_t)<1$. The
standard repair is a \emph{common} quadratic Lyapunov function: find
$P\succ0$ with $\Phi^\top P\Phi\preceq\rho^2P$ simultaneously for every
$\Phi$ in the family. If it exists, $\lVert P^{1/2}\Phi P^{-1/2}\rVert\le\rho$
is a \emph{norm} bound; the operator norm is convex, so every convex mixture
inherits it, and submultiplicative, so arbitrary products do too. That single
object would close the mixture, the gated case and the transients at once.

We solved it. One technical point makes the vertex argument valid: $\Phi$ is
not affine in $(\omega_0,\zeta,c)$ ---the entries carry $\omega_0^2$,
$\zeta\omega_0$ and $c^2$--- so we reparameterize in
$(\Delta t^2\omega_0^2,\,2\Delta t\zeta\omega_0,\,\Delta t^2c^2,\dots)$, in
which it \emph{is} affine; the derived box contains the physical one, so the
certificate is conservative rather than optimistic. Feasibility is re-checked
by explicit eigenvalue tests rather than trusted from the solver status.

\paragraph{What it buys.} On the envelope of the trained checkpoints
($\omega_0\in[0.488,0.529]$, $\zeta\in[0.475,0.538]$, $c\le0.371$, read off
the weights) a common $P$ exists with $\rho=0.952$ and
$\kappa=\sqrt{\lambda_{\max}/\lambda_{\min}}=2.51$; the worst vertex spectral
radius is $0.732$, so the trained models sit well inside. This is strictly
stronger than the runtime probe: a norm bound, valid for mixtures and for
gated products \emph{within the wave branch}.

\paragraph{Where it stops, in three places.} First, the \emph{declared}
coefficient box is not certifiable because it is not stable: $40$ of its $64$
vertices are already divergent (worst $\rho=5.24$), the culprit being $c$,
which enters the stiffness as $c^2\lambda_{\max}(\mathbf{L})$ and at
$c_{\max}=0.7$ already contributes $\approx1.83$. The certifiable frontier
degenerates to $c=0,\ \omega_0\le0.276$. What keeps the model away from those
configurations is training, not the box; the honest reading is that the caps
in the configuration are nominal and too loose, and the operative guarantee
remains the runtime check.

Second, the mixture still does not close: including the diffusive branch and
the antisymmetric operators, $\max\lVert\Phi\rVert_P=1.57$ under the wave
branch's $P$, and no common $P$ exists for the enlarged set either.

Third, and most consequentially, the closed loop does not close by
small-gain. The condition $\rho+\kappa\Delta t\,L_F<1$ requires
$L_F<0.0192$. The only state-dependent forcing is
$f_{\mathrm{gain}}f_\theta([h;s])$, whose Lipschitz constant in $h$ is
bounded exactly from the weights by
$f_{\mathrm{gain}}\lVert W_2\rVert\,\mathrm{Lip}(\mathrm{SiLU})\lVert W_{1h}\rVert$;
measured on the trained checkpoints it lies between $0.367$ and $0.918$.
The gap is a factor of roughly fifty, and it is structural rather than
marginal: the contraction margin $1-\rho=0.048$ is an order of magnitude
smaller than the learned forcing gain. Certifying a single operating point
instead of the envelope would raise the threshold to $\approx0.13$ and still
leave a factor of seven.

\paragraph{What that does and does not mean.} Small-gain is a
\emph{sufficient} condition. Its failure does not show the loop is unstable
---it does not diverge in training--- but that this technique does not reach.
We also do not attempt the loop through the host
($h\to$ modulation $\to$ transformer $\to$ interoception $\to f_\theta$),
which would require a Lipschitz bound on the host with respect to its own
modulation that we do not have. We state the boundary rather than blur it:
the certificate is an \emph{integrator} certificate, extended here from
spectral radius to norm on the region actually used, and it stops at the
mixture and at the closed loop. The scripts are
\texttt{experiments/certify\_lmi.py}, \texttt{experiments/certify\_lmi\_ckpt.py} and
\texttt{experiments/certify\_lmi\_iss.py}.

\bibliographystyle{tmlr}
\bibliography{refs}

\end{document}